%% file: neurips_2026.tex
\documentclass{article}

 \usepackage[preprint]{neurips_2026}

\usepackage[utf8]{inputenc} 
\usepackage[T1]{fontenc}    
\usepackage{hyperref}       
\usepackage{url}            
\usepackage{booktabs}       
\usepackage{amsfonts}       
\usepackage{nicefrac}       
\usepackage{microtype}      
\usepackage{xcolor}         
\usepackage{amsmath}
\usepackage{enumitem}
\usepackage{graphicx}
\usepackage{multirow}
\usepackage{subcaption}
\usepackage{wrapfig}

\usepackage{silence}
\usepackage{pifont}
\newcommand{\cmark}{\ding{51}}
\newcommand{\xmark}{\ding{55}}

\usepackage{xspace} 
\newcommand{\ours}{\texttt{ConceptM\textsuperscript{3}oE}\xspace}
\newcommand{\oursfull}{\textbf{Concept} \textbf{M}ulti\textbf{m}odal \textbf{MoE}\xspace}

\newcommand{\ItwoMoE}{\texttt{I\textsuperscript{2}MoE}\xspace}

\makeatletter
\NewEnviron{arxivhide}{
  \if@preprint
  \else
    \BODY
  \fi
}
\makeatother

\title{ConceptM\textsuperscript{3}oE: Concept-Guided Multimodal Mixture of Experts for Interpretable Computational Pathology}

\author{%
  \textbf{Xuan Wang\textsuperscript{1,$*$}, Zhongling Xu\textsuperscript{1,$*$}, Gopi Kannedhara\textsuperscript{1}, Joakim Nguyen\textsuperscript{1}, Jian Yu\textsuperscript{1}, Jinrui Fang\textsuperscript{1},} \\
  \textbf{Abdurrahmaan Baghdadi\textsuperscript{1}, Tianlong Chen\textsuperscript{2}, Awais Naeem\textsuperscript{1}, Chandra Krishnan\textsuperscript{3}, Edward Castillo\textsuperscript{1},} \\
  \textbf{Andrew H. Song\textsuperscript{4}, Ankita Shukla\textsuperscript{5}, Ying Ding\textsuperscript{1}, Nicholas Konz\textsuperscript{2,$\dagger$}, Hairong Wang\textsuperscript{1,$\dagger$}} \\
  \vspace{0.2cm} \\
  \textsuperscript{1} University of Texas at Austin, Austin, TX, USA \\
  \textsuperscript{2} University of North Carolina at Chapel Hill, NC, USA \\
  \textsuperscript{3} Dell Children's Medical Center, Austin, TX, USA \\
  \textsuperscript{4} The University of Texas MD Anderson Cancer Center, Houston, TX, USA \\
  \textsuperscript{5} University of Nevada, Reno, NV, USA \\
  \vspace{0.1cm} \\
  \texttt{\{xuan.wang, xuzhl\}@austin.utexas.edu} \\
  \texttt{nickk124@cs.unc.edu, hairong.wang@austin.utexas.edu}
}

\begin{document}

\begingroup
\renewcommand{\thefootnote}{}
\footnotetext{$*$ Equal contribution. $\dagger$ Corresponding authors.}
\endgroup

\maketitle

\begin{abstract}
\vspace{-0.5em}

Healthcare models are transitioning from unimodal prediction toward multimodal reasoning over heterogeneous diagnostic inputs. In computational pathology, for complex tumor subtypes where morphology alone can be challenging to distinguish, pathology reports and molecular measurements may provide additional diagnostic evidence alongside whole-slide images, yet existing models often fail to clarify how diverse signals assemble into recognizable diagnostic concepts.
We propose \ours{} (\oursfull{}), which embeds concept formation directly within interaction-aware mixture-of-experts (MoE) pathways. 
The architecture decomposes evidence into modality-specific, redundant, and synergistic experts, which are then projected into structured concept bottlenecks mapping latent features to a hierarchy of morphology and biomarker concepts. 
To prevent the information loss typical of interpretable bottlenecks, we utilize residual pathways within each expert to allow task-relevant signals to flow both through the concepts and directly to the final task prediction, so that high performance is maintained alongside interpretability. 
Across an institutional pediatric brain tumor cohort and a public glioma cohort, the framework delivers competitive performance to unconstrained models while producing reasoning traces validated by an independent neuropathologist. 
In data-limited regimes, \ours{} improves limited-data performance, increasing macro-F1 from 56.41\% to 66.70\% at small training sizes compared to non-concept-informed baselines, while also showing faster training convergence consistent with the regularizing effect of concept learning. 
This work offers a scalable path toward high-performance medical AI that is inherently verifiable and better aligned with the complex decision-making of clinical practice.

\end{abstract}
\vspace{-0.5em}

\input{intro.tex}

\input{method.tex}

\input{experiment.tex}

\input{conclusion.tex}

\begin{ack}
This work was supported by the NSF AI Institute for Foundations of Machine Learning and partially funded by the National Institutes of Health (NIH) under award 1R01EB03710101. The computational experiments were conducted using the MIATA server at the School of Information at The University of Texas at Austin, as well as the Lonestar6 and Vista GPU clusters provided by the Texas Advanced Computing Center (TACC) and the Center for Generative AI (CGAI) at The University of Texas at Austin. The views and conclusions contained in this document are those of the authors and should not be interpreted as representing the official policies, either expressed or implied, of the NIH or the U.S. Government.
\end{ack}

\bibliographystyle{plain}
\bibliography{conceptm3oe_references}


\begin{arxivhide}
\newpage
\input{checklist.tex}
\end{arxivhide}


\newpage
\input{appendix.tex}


\end{document}

%% file: intro.tex
\vspace{-0.5em}
\section{Introduction}
\vspace{-0.5em}
\label{sec:intro}

Data-driven models have revolutionized diagnostic and prognostic workflows in healthcare, driven by progress in medical image analysis, clinical outcome prediction, biomedical foundation models, and multimodal diagnostic systems \citep{acosta2022multimodal, moor2023foundation, tu2024towards}. A primary objective of these systems is the synthesis of heterogeneous patient data to facilitate reliable clinical decision-making. Computational pathology has emerged as a critical domain for these methodologies, as whole-slide images (WSIs) provide high-resolution tissue-level evidence for cancer diagnosis, prognosis, and biomarker identification \citep{lu2021data, chen2024uni, vorontsov2024virchow, huang2024visual, ding2024multimodal}. However, definitive pathological diagnosis is seldom predicated on image-based patterns alone. In practice, pathologists synthesize evidence from morphology, cellular architecture, and molecular or biomarker profiles, an integrative approach that is critical in oncology for distinguishing clinically distinct tumor subtypes \citep{lipkova2022artificial, boehm2022harnessing, chen2020pathomic, sturm2016new, mackay2017integrated, petralia2020integrated, pfister2022molecular}.

\begin{figure}[t]
\centering
\includegraphics[width=\columnwidth]{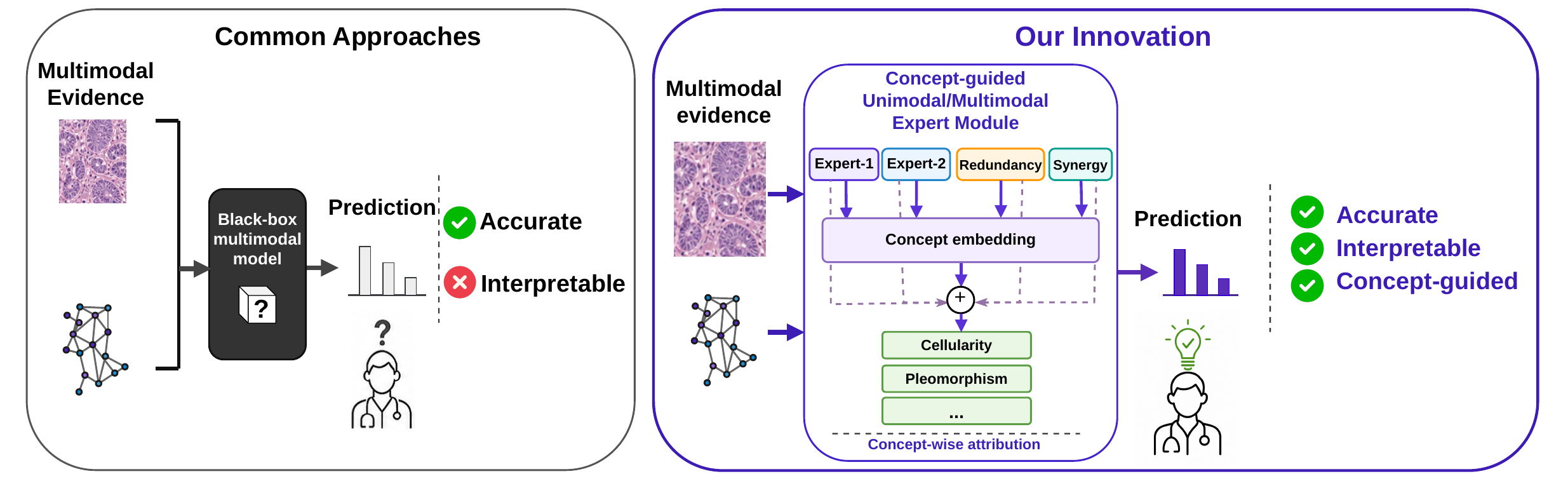}
\caption{\textbf{Illustrative comparison.}
Left: Common multimodal models produce accurate but opaque predictions from WSI and cell-graph evidence.
Right: Our concept-guided expert module links predictions to morphology concepts with unimodal and cross-modal attribution. Predictions are informed by concepts attributed separately to unimodal and cross-modal evidence.}
\label{fig:teaser}
\vspace{-1.5em}
\end{figure}

The complexity of cancer diagnosis necessitates multimodal learning to integrate images, clinical variables, and molecular data \citep{baltruvsaitis2019multimodal, liang2024survey, lahat2015multimodal, acosta2022multimodal, moor2023foundation, tu2024towards}. Prior research has established various fusion architectures, ranging from early and late fusion to tensor based, low-rank and attention driven approaches to model cross-modal interactions \citep{ngiam2011multimodal, zadeh2017tensor, liu2018efficient, nagrani2021attention, xu2023multimodal, chen2022multimodal}. In oncology, these methodologies have been applied to integrate histopathology, genomics, and EHR data for clinical outcome prediction and biomarker discovery \citep{huang2020fusion, stahlschmidt2022multimodal, lipkova2022artificial, boehm2022harnessing, chen2020pathomic, vanguri2022multimodal}. Mixture-of-experts (MoE) architectures offer further flexibility by specializing pathways to specific input sources or interaction patterns \citep{jacobs1991adaptive, jordan1994hierarchical, shazeer2017outrageously, mustafa2022multimodal, cai2024survey}. Recent work has explicitly modeled modality-specific, redundant, and synergistic expert interactions to improve transparency via expert reweighting \citep{xin2025i2moe}. However, clinical reliability requires not only correct information weighting but also evidence utilization consistent with pathological reasoning. Useful explanations should demonstrate how evidence from each source contributes to recognizable diagnostic features before affecting the final prediction, a level of detail current multimodal methods seldom provide.

Concept-based learning grounds model predictions in human-interpretable intermediate variables. Concept bottleneck models (CBM) enable inspection by predicting concepts prior to the final label \citep{koh2020concept}, while concept embedding models (CEM) utilize high-dimensional representations to mitigate accuracy-interpretability trade-offs \citep{espinosa2022concept}. Subsequent research has introduced functional variants including post-hoc, probabilistic, and language-guided models \citep{yuksekgonul2023posthoc, oikarinen2023label, kim2023probcbm, yang2023labo, srivastava2024vlgcbm, chauhan2023interactive}, alongside structured concept organizations \citep{de2026mixture, panousis2023coarse, xie2026hierarchical, hill2026hierarchical, wang2026conceptflow}. In parallel, computational pathology has progressed through weakly supervised foundation models \citep{lu2021data, chen2024uni, vorontsov2024virchow}, vision-language architectures utilizing pathology reports \citep{huang2024visual, lu2024visual, ding2024multimodal}, and multimodal integration of histopathology with omics data \citep{chen2020pathomic, chen2022multimodal, steva2022multimodal, vanguri2022multimodal, nguyen2026clinically, yu2026pathmoe}. However, while these pathology frameworks achieve strong performance, they typically do not employ concept-guided reasoning, relying instead on high-dimensional fusion that lacks semantic transparency. Consequently, explanations remain limited to slide or modality-level importance, failing to bridge the gap between multimodal evidence and recognizable clinical reasoning.

To address these challenges, we introduce \ours{} (\oursfull), a concept-informed multimodal prediction model which establishes a structured semantic interface between multimodal evidence and diagnostic decisions. An overview of \ours{} is shown in Fig.~\ref{fig:teaser}. Unlike standard fusion models, \ours{} embeds concept learning directly within an interaction-aware Mixture-of-Experts (MoE) architecture. Each expert, specialized either for unimodal evidence or for capturing redundant or synergistic cross-modal interactions, is internally structured as a CEM that maps latent features to clinical descriptors. These expert-specific representations are then organized into a two-level hierarchy where morphology concepts facilitate the learning of biomarker and molecular concepts, reflecting the clinical dependency between histological patterns and genetic status. To mitigate the performance degradation typically associated with interpretable bottlenecks, each expert utilizes a residual pathway to allow task-relevant signals to flow both through the human-understandable concepts and directly to the final task prediction, ensuring the model maintains high predictive power alongside transparency. Across two pediatric brain tumor cohorts comprising an institutional dataset and the public TCGA-GBM dataset, \ours{} achieves performance competitive with unconstrained models while producing modality attribution and concept-informed reasoning traces validated by an independent neuropathologist. Notably, in data-limited regimes ($N=50$), the framework demonstrates significant sample efficiency, increasing macro-F1 from $56.41\%$ to $66.70\%$ and exhibiting highly accelerated training convergence due to the regularizing effect of the structured concept bottleneck. 

Our contributions are summarized as follows. \textbf{(1) Multimodality-aware concept representation learning:} We formulate morphology concepts as traceable multimodal representations decomposed into modality-specific and cross-modal evidence. 
When biomarker labels are available, we further extend the model into a hierarchy where morphology concepts inform biomarker-level concepts that are less directly observable from histology but more closely tied to diagnosis. \textbf{(2) Multilevel interpretability and clinical validation:} We provide a framework for verifiable clinical reasoning by decomposing diagnostic features into modality-specific and cross-modal interaction-aware (redundant or synergistic) expert evidence. 
\textbf{(3) Competitive performance and sample efficiency:} Across institutional and public (TCGA-GBM) pediatric brain tumor cohorts, \ours{} is competitive with the performance of unconstrained concept-blind models while providing inherent interpretability. 

\vspace{-0.5em}

%% file: method.tex
\section{Method}
\vspace{-0.5em}
\label{sec:method}
\ours{} (illustrated in Fig.~\ref{fig:overview}) is motivated by a diagnostic reasoning question: can a multimodal model
organize histopathology-derived input modalities into a transparent pathway from
evidence sources to clinically meaningful concepts, then to final diagnosis? In this work, we use WSI and cell-graph as model inputs. Pathology reports and biomarker annotations are used to define the morphology and biomarker concept layers. A pathology concept can be supported by a single source of evidence, by information
shared across sources, or by synergistic evidence from multiple sources. For example, cellularity, pleomorphism, mitotic activity, necrosis, and Rosenthal fibers are morphology concepts extracted from pathology reports, while GFAP, Synaptophysin, INI1, H3K27M, and ALK1 are biomarker concepts obtained from available protein-level or molecular annotations.
At the diagnostic level, clinically meaningful interpretability requires exposing the concept-to-label
mapping, clarifying how morphology and biomarker concepts provide intermediate
clinical evidence for the final prediction. \ours{} is designed to learn this reasoning process by embedding concept
formation within an interaction-aware mixture-of-experts architecture. The
resulting representation exposes how each concept is
predicted from single-source signals, redundant signals shared across sources,
or synergistic signals jointly provided by multiple sources, and how these
concepts provide intermediate clinical evidence for the final diagnostic
prediction.

\begin{figure}[t]
\centering
\includegraphics[width=0.98\textwidth]{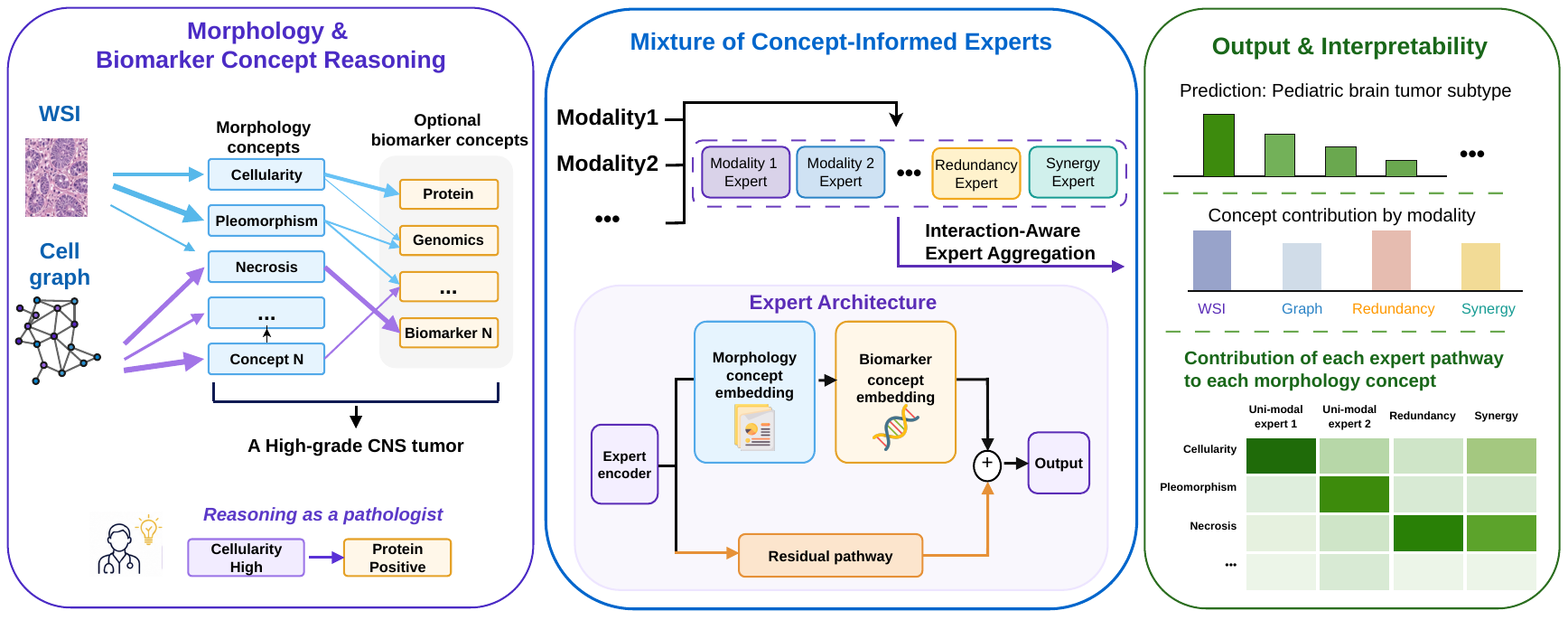}
\caption{\textbf{Overview of the proposed \ours{} framework.}
Left: WSI and cell-graph inputs are processed by unimodal, redundancy, and synergy experts, with each expert combining concept embedding and residual pathways. 
Middle: morphology concepts are learned from multimodal evidence and can further support optional biomarker concepts, forming a clinically meaningful reasoning layer. 
Right: the model outputs pediatric brain tumor subtype predictions and provides modality- and expert-level interpretability through concept-wise evidence attribution.}
\label{fig:overview}
\vspace{-1.5em}
\end{figure}
\vspace{-0.4em}
\subsection{Problem formulation}
\label{sec:problem_formulation}
\vspace{-0.4em}

Let
$
\mathcal{D}
=
\{(\mathbf{X}_n^{(1)},\ldots,\mathbf{X}_n^{(M)},y_n)\}_{n=1}^{N}
$
be a multimodal dataset with $M$ modalities and $C$ diagnostic classes, where $\mathbf{X}_n^{(m)}$ is modality $m$ of sample $n$ and $y_n\in\{1,\ldots,C\}$. 
We use $M=2$ histopathology-derived modalities: WSI patch features, capturing local tissue and cytologic patterns, and cell-graph features, capturing cellular organization and neighborhood structure. 
Each sample is annotated with a two-level concept hierarchy from pathology reports and biomarker annotations: $K_1$ morphology concepts with soft-ordinal targets $\mathbf{t}^{1}\in[0,1]^{K_1}$, and $K_2$ biomarker concepts with binary targets $\mathbf{t}^{2}\in\{0,1\}^{K_2}$. 
Appendix~\ref{app:concept} details the vocabulary, extraction prompts, and masking rules.

\vspace{-0.4em}
\subsection{Interaction-decomposed expert evidence}
\label{sec:interaction_experts}
\vspace{-0.4em}

Each modality is encoded as 
$\mathbf{e}_m=h_m(\mathbf{X}^{(m)})\in\mathbb{R}^{d}$, $m=1,\ldots,M$, 
using a modality-specific encoder $h_m$. 
In our histopathology setting, $M=2$: WSI patch features are aggregated by gated-attention MIL pooling~\cite{lu2021data}, while cell-graph features, constructed from the same slide to encode cellular organization and neighborhood structure, are encoded by GraphSAGE-style message passing followed by attention pooling~\cite{hamilton2017inductive}. 
Following the interaction semantics of \ItwoMoE{}~\cite{xin2025i2moe}, \ours{} defines $E=M+2$ expert pathways,
$
\mathcal{E}=\{U_1,\ldots,U_M,R,S\},
$
where $U_m$, $R$, and $S$ respectively capture modality-unique, cross-modal redundant, and cross-modal synergistic evidence. 
Here, redundancy denotes diagnostic cues consistently reflected by both WSI and cell-graph representations, whereas synergy denotes cues that become informative only when local tissue appearance and cellular architecture are integrated. 
Each expert outputs
$
\mathbf{z}_e=f_e([\mathbf{e}_1;\ldots;\mathbf{e}_M])\in\mathbb{R}^{d},
$
$e\in\mathcal{E}$.

Expert semantics are enforced through a perturbation loss. 
Let $\mathbf{p}_e^{\mathrm{clean}}$ denote the prediction of expert $e$ from the original modality embeddings, and $\mathbf{p}_e^{m}$ its prediction after injecting Gaussian noise into $\mathbf{e}_m$. 
We define
{\setlength{\abovedisplayskip}{1pt}
\setlength{\belowdisplayskip}{1pt}
\setlength{\abovedisplayshortskip}{1pt}
\setlength{\belowdisplayshortskip}{1pt}
\begin{equation}
  \mathcal{L}_{\mathrm{int}}^e
  =
  \sum_{m=1}^{M}
  s_e^m
  D_{\mathrm{KL}}
  \left(
    \mathbf{p}_e^{\mathrm{clean}}\,\|\,\mathbf{p}_e^{m}
  \right),
  \label{eq:interaction_loss}
\end{equation}}
where $s_e^m=-1$ if expert $e$ is encouraged to be sensitive to modality $m$, and $s_e^m=+1$ if it is encouraged to be invariant. 
Thus, $U_m$ is sensitive only to modality $m$, $R$ is invariant to single-modality perturbations, and $S$ is sensitive to all modalities. Detailed perturbation and sign conventions are provided in Appendix~\ref{app:interaction_details}.
\vspace{-0.4em}
\subsection{Interaction-aware concept formation}
\label{sec:concept_formation}

The key difference from a standard concept bottleneck is that concepts are formed \emph{inside each expert}. For expert $e$ and concept $k$, we compute positive and negative concept states from the expert latent:
\begin{equation}
  \mathbf{c}_{e,k}^{+} = \mathrm{LeakyReLU}(\phi_{e,k}^{+}(\mathbf{z}_e)), \qquad
  \mathbf{c}_{e,k}^{-} = \mathrm{LeakyReLU}(\phi_{e,k}^{-}(\mathbf{z}_e)), \qquad
  \mathbf{c}_{e,k}^{\pm}\in\mathbb{R}^{d_c}.
  \label{eq:positive_negative_concepts}
\end{equation}
For ordinal morphology concepts, such as necrosis or pleomorphism, the positive and negative states encode evidence for higher versus lower concept expression, allowing both the presence and relative absence of clinically meaningful findings to inform diagnosis. 
The scalar concept activation is $p_{e,k} = \sigma\bigl(s_{e,k}([\mathbf{c}_{e,k}^{+};\mathbf{c}_{e,k}^{-}])\bigr) \in[0,1]$, which is supervised by the corresponding concept target. The concept representation passed to the task head is
\begin{equation}
  \hat{\mathbf{c}}_{e,k} = p_{e,k}\mathbf{c}_{e,k}^{+} + (1-p_{e,k})\mathbf{c}_{e,k}^{-} + \gamma_{\mathrm{res}}\,\psi_{e,k}(\mathbf{z}_e),
  \label{eq:soft_concept_embedding}
\end{equation}
where $\psi_{e,k}:\mathbb{R}^{d}\rightarrow\mathbb{R}^{d_c}$ is a residual projection and $\gamma_{\mathrm{res}}\in\{0,1\}$ controls the residual pathway. Here, $p_{e,k}$ provides the readable concept value, while $\hat{\mathbf{c}}_{e,k}$ preserves within-concept variation and residual evidence beyond the predefined vocabulary, allowing concepts to regularize expert representations without collapsing prediction to a hard list of scalar probabilities.

\textbf{Hierarchical concepts.}
When biomarker annotations are available, concepts are organized into two levels. The first level contains morphology concepts grounded in histologic evidence; the second level contains protein-level and molecular biomarker concepts, which are not directly observable but can be predicted from the expert latent representation and the first-level morphology concepts. 
This hierarchy reflects brain tumor subtyping practice, where H\&E morphology summarizes observable tissue findings and biomarkers provide additional context for subtype refinement. First-level morphology concept embeddings are collected as $\mathbf{B}_{e,1} = [\hat{\mathbf{c}}_{e,1};\ldots;\hat{\mathbf{c}}_{e,K_1}] \in\mathbb{R}^{K_1d_c}$, where each $\hat{\mathbf{c}}_{e,k}$ includes the per-concept residual $\psi_{e,k}(\mathbf{z}_e)$ defined in Eq.~\eqref{eq:soft_concept_embedding}. 
The second-level input concatenates the expert latent with the first-level concept embeddings: $\mathbf{z}_{e,2} = [\mathbf{z}_e;\mathbf{B}_{e,1}] \in \mathbb{R}^{d + K_1 d_c}$. Each biomarker block independently maps $\mathbf{z}_{e,2}$ to $\hat{\mathbf{c}}_{e,K_1+j}$ using the same per-concept residual mechanism as Eq.~\eqref{eq:soft_concept_embedding}, yielding
\begin{equation}
  \mathbf{B}_{e,2} = [\hat{\mathbf{c}}_{e,K_1+1};\ldots;\hat{\mathbf{c}}_{e,K_1+K_2}] \in\mathbb{R}^{K_2d_c}.
  \label{eq:l2_concepts}
\end{equation}
Thus, morphology-level evidence can inform biomarker-level concept prediction while the raw expert latent remains available. The expert logit is $\ell_e = g_e([\mathbf{B}_{e,1};\mathbf{B}_{e,2}]) \in\mathbb{R}^{C}$. If $K_2=0$, the second level is omitted and $\ell_{e}=g_e(\mathbf{B}_{e,1})$.

\textbf{Soft bottleneck preserves the raw expert predictor.}
Eq.~\eqref{eq:soft_concept_embedding} augments each supervised concept with an embedding and residual projection, providing clinically organized concept coordinates while retaining access to the expert latent. This is useful in pathology, where predefined concepts may only partially cover the visual evidence relevant to diagnosis; the residual pathway preserves complementary signal beyond explicitly supervised concepts. The following proposition formalizes this soft-bottleneck property.

For a fixed expert, let $\mathbf{r}_{e,1} = [\psi_{e,1}(\mathbf{z}_e);\ldots;\psi_{e,K_1}(\mathbf{z}_e)]$ and define the representation chain:
\begin{equation}
  R_{e,0} = \mathbf{z}_e, \qquad R_{e,1} = [\mathbf{z}_e;\mathbf{B}_{e,1};\mathbf{r}_{e,1}], \qquad R_{e,2} = [\mathbf{z}_e;\mathbf{B}_{e,1};\mathbf{r}_{e,1};\mathbf{B}_{e,2}].
  \label{eq:soft_concept_rep_chain}
\end{equation}

\begin{proposition}[Soft concept augmentation preserves the raw expert predictor]
\label{prop:soft_concept_augmentation}
Let $Y$ be the class label and $\mathcal R_{\log}^{\star}(R) = \inf_q \mathbb E[-\log q(Y\mid R)]$ be the optimal population log-loss from representation $R$. The representations in Eq.~\eqref{eq:soft_concept_rep_chain} satisfy $I(R_{e,0};Y) \le I(R_{e,1};Y) \le I(R_{e,2};Y)$ and $\mathcal R_{\log}^{\star}(R_{e,2}) \le \mathcal R_{\log}^{\star}(R_{e,1}) \le \mathcal R_{\log}^{\star}(R_{e,0})$.
Moreover, if $\mathcal F_j$ denotes the task-head class on $R_{e,j}$ and the head can zero out appended coordinates, then $\mathcal F_0 \subseteq \mathcal F_1 \subseteq \mathcal F_2$. Let $\widehat{\mathcal L}_{\mathrm{cls}}(f;R_{e,j})$ be the empirical classification loss of head $f$ on representation $R_{e,j}$. Then
\begin{equation*}
  \inf_{f\in\mathcal F_2} \widehat{\mathcal L}_{\mathrm{cls}}(f;R_{e,2}) \le \inf_{f\in\mathcal F_1} \widehat{\mathcal L}_{\mathrm{cls}}(f;R_{e,1}) \le \inf_{f\in\mathcal F_0} \widehat{\mathcal L}_{\mathrm{cls}}(f;R_{e,0}).
\end{equation*}
If the appended concept and residual coordinates are deterministic functions of $\mathbf{z}_e$ and $\mathbf{z}_e$ is retained, then the information guarantee becomes exact preservation: $I(R_{e,j};Y)=I(\mathbf{z}_e;Y)$ for $j\in\{0,1,2\}$.
\end{proposition}
We evaluate this property through the information-plane analysis in Figs.~\ref{fig:mi_plane} and~\ref{fig:mi_curve}, where concept-augmented representations retain label-relevant dependence rather than collapsing into a lossy scalar bottleneck. The proof is in Appendix~\ref{app:proof_soft_concept_augmentation}.

\subsection{Training objective and statistical interpretation}
\label{sec:training_objective}

The model is optimized with classification, concept supervision, and interaction specialization: $\mathcal{L} = \mathcal{L}_{\mathrm{cls}} + \lambda_1\mathcal{L}_{\mathrm{concept}}^{\mathrm{L1}} + \lambda_2\mathcal{L}_{\mathrm{concept}}^{\mathrm{L2}} + \lambda_{\mathrm{int}}\mathcal{L}_{\mathrm{int}}$, where $\mathcal{L}_{\mathrm{cls}}$ is class-weighted cross-entropy, $\mathcal{L}_{\mathrm{int}} = |\mathcal{E}|^{-1} \sum_{e\in\mathcal{E}}\mathcal{L}_{\mathrm{int}}^e$, and $\lambda_1,\lambda_2,\lambda_{\mathrm{int}}$ are hyperparameters. Concept supervision is applied only to observed concepts:
\begin{equation}
  \mathcal{L}_{\mathrm{concept}}^{\mathrm{L}q} = \frac{1}{|\mathcal{E}|} \sum_{e\in\mathcal{E}} \frac{\sum_{n\in\mathcal{B}}\sum_{k\in\mathcal{K}_q} m_{n,k}^{q}\, \ell_q(p_{e,n,k},t_{n,k}^{q})}{\sum_{n\in\mathcal{B}}\sum_{k\in\mathcal{K}_q} m_{n,k}^{q} + \varepsilon}, \qquad q\in\{1,2\},
  \label{eq:concept_loss}
\end{equation}
where $\mathcal{B}$ is a minibatch, $\mathcal{K}_q$ is the concept index set at level $q$, $m_{n,k}^{q}$ indicates whether concept $k$ is observed for sample $n$, $\ell_q$ is the level-specific concept loss, and $\varepsilon$ prevents division by zero.

\textbf{Concept supervision localizes the predictor class.}
Let $\mathcal H$ be the predictor class induced by the architecture, $\ell_{\mathrm{cls}}(h(X),Y)$ the per-sample classification loss, and $\mathcal R_{\mathrm{cls}}(h) = \mathbb E[\ell_{\mathrm{cls}}(h(X),Y)]$ and $\widehat{\mathcal R}_{\mathrm{cls}}(h) = n^{-1}\sum_{i=1}^{n}\ell_{\mathrm{cls}}(h(X_i),Y_i)$ its population and empirical risks. Let $\mathcal R_{\mathrm{concept}}^{\mathrm{L}q}(h)$ denote the population counterpart of Eq.~\eqref{eq:concept_loss}, and define $\mathcal C(h) = \lambda_1\mathcal R_{\mathrm{concept}}^{\mathrm{L1}}(h) + \lambda_2\mathcal R_{\mathrm{concept}}^{\mathrm{L2}}(h)$. For $\tau\ge0$, consider the localized class $\mathcal H_\tau = \{h\in\mathcal H:\mathcal C(h)\le\tau\}$. The concept terms in the training objective can be viewed as a Lagrangian relaxation of this constrained formulation, guiding the model toward clinically aligned representations while retaining the diagnostic objective.

\begin{theorem}[Concept-aligned localization]
\label{thm:concept_localization}
Assume $0\le\ell_{\mathrm{cls}}(h(X),Y)\le B_{\ell}$ for all $h\in\mathcal{H}_{\tau}$, and let $\hat h_{\tau} \in \arg\min_{h\in\mathcal{H}_{\tau}} \widehat{\mathcal R}_{\mathrm{cls}}(h)$. Then, with probability at least $1-\delta$,
\begin{equation*}
  \mathcal R_{\mathrm{cls}}(\hat h_{\tau}) \le \inf_{h\in\mathcal H_{\tau}}\mathcal R_{\mathrm{cls}}(h) + 4\mathfrak R_n(\ell_{\mathrm{cls}}\circ\mathcal H_{\tau}) + 2B_{\ell}\sqrt{\frac{\log(2/\delta)}{2n}},
\end{equation*}
where $\mathfrak R_n(\ell_{\mathrm{cls}}\circ\mathcal H_{\tau})$ is the empirical Rademacher complexity. Since $\mathcal H_{\tau}\subseteq\mathcal H$, $\mathfrak R_n(\ell_{\mathrm{cls}}\circ\mathcal H_{\tau}) \le \mathfrak R_n(\ell_{\mathrm{cls}}\circ\mathcal H)$. If $\mathcal H_{\tau}$ is nearly task-sufficient, i.e., $\inf_{h\in\mathcal H_{\tau}}\mathcal R_{\mathrm{cls}}(h) \le \inf_{h\in\mathcal H}\mathcal R_{\mathrm{cls}}(h) + \epsilon_{\mathrm{app}},$ then localization improves the finite-sample bound whenever the complexity reduction exceeds $\epsilon_{\mathrm{app}}$.
\end{theorem}

Theorem~\ref{thm:concept_localization} shows that clinical concepts can localize the search space while preserving diagnostic predictors. We evaluate this in the subsampling experiment in Fig.~\ref{fig:sample_efficiency}; the proof is in Appendix~\ref{app:proof_concept_localization}.

\textbf{Optimization implication.} The same objective also introduces auxiliary gradients from concept supervision
and interaction specialization. Proposition~\ref{prop:aux_gradient_contraction}
in Appendix~\ref{app:proof_aux_gradient_contraction} shows that, under
smoothness, Polyak-\L{}ojasiewicz (PL) condition, and gradient-alignment conditions, the joint update
contracts the classification loss faster than the task-only update. This explains Fig.~\ref{fig:loss_curves}: faster convergence is expected when the
auxiliary gradients align with the diagnostic objective. 
\vspace{-0.5em}

%% file: experiment.tex
\section{Experiments}
\label{sec:experiments}
\vspace{-0.5em}

We evaluate \ours{} across four key dimensions: (1) predictive performance on par with state-of-the-art multimodal baselines (Sec.~\ref{sec:main_results}); (2) the clinical validity and interpretability of the learned concept activations and expert gates (Sec.~\ref{sec:interpretability}); (3) predictive robustness in data-limited regimes and training convergence dynamics (Sec.~\ref{sec:sample_efficiency}); and (4) the preservation and semantic reorganization of task-relevant information, analyzed via the information plane (Sec.~\ref{sec:mi_analysis}). Extensive architectural ablation studies are deferred to Appendix~\ref{app:ablation}.


\vspace{-0.4em}
\subsection{Experimental Setup}
\label{sec:exp_setup}
\vspace{-0.4em}
\paragraph{Datasets.}
We evaluate on an institutional pediatric brain tumor cohort (PBT) from Dell Children's Medical Center and a public adult glioma cohort (TCGA) ~\cite{heath2021nci, cancer2008comprehensive, cancer2015comprehensive}. 
Central nervous system (CNS) tumors provide a clinically relevant setting because complex subtyping often relies on morphology together with immunohistochemical or molecular evidence. PBT is primary cohort, while TCGA provides an external cohort with related CNS tumor classes for assessing cross-cohort generalization.
PBT contains 199 diagnostic WSIs from 196 patients across 4 classes: Ependymoma, High-grade CNS, Low-grade CNS, and Non-glial. 
Each sample provides WSI patches (UNIv2~\cite{chen2024uni}), a cell graph, and pathology report text. Concept annotations are LLM-extracted from reports: $K_1$ morphology concepts (soft ordinal in $[0,1]$) and $K_2$ IHC biomarker concepts (binary labels, for PBT only). 
The TCGA dataset contains 208 WSIs from 67 patients across 2 classes: High-grade CNS and Low-grade CNS.

\textbf{Evaluation Protocol.} All models are evaluated via cross-validation with strict patient-level splits to prevent data leakage. The primary evaluation metric is macro-averaged F1. We also report per-class precision and recall. 
\begin{table}[!htbp]
\vspace{-1em}
\setlength{\tabcolsep}{4pt}
\captionsetup{font=small,skip=2pt}
\caption{Backbone comparison on PBT and TCGA datasets. \ours{} is the flat (morph) variant. ``HG/LG''= HG/LG-CNS. ``NGT'' = Non-glial tumor. \textbf{Bold}: best, \underline{underline}: second best.}\label{tab:per_class_metrics_main}
\vspace{-0.2em}
\centering
\resizebox{\linewidth}{!}{%
\begin{tabular}{l cccc ccc cc ccc}
\toprule
& \multicolumn{7}{c}{PBT} & \multicolumn{5}{c}{TCGA} \\
\cmidrule(lr){2-8} \cmidrule(lr){9-13}
& \multicolumn{4}{c}{Per-class F1} & \multicolumn{3}{c}{Macro} & \multicolumn{2}{c}{Per-class F1} & \multicolumn{3}{c}{Macro} \\
\cmidrule(lr){2-5} \cmidrule(lr){6-8} \cmidrule(lr){9-10} \cmidrule(lr){11-13}
Model & Epend. & LG & HG & NGT & P & R & F1 & LG & HG & P & R & F1 \\
\midrule
Early Fusion & \textbf{0.600} & 0.844 & 0.742 & 0.732 & 0.746 & 0.735 & 0.729 & 0.810 & 0.815 & 0.839 & \underline{0.825} & 0.812 \\
MoEPP        & \underline{0.400} & 0.864 & \underline{0.792} & 0.737 & 0.728 & 0.703 & 0.698 & 0.816 & \underline{0.816} & \underline{0.857} & 0.824 & 0.816 \\
Switch Gate  & 0.200 & 0.859 & 0.775 & 0.773 & 0.673 & 0.662 & 0.652 & 0.782 & 0.815 & 0.833 & 0.808 & 0.799 \\
PathMoE      & \textbf{0.600} & \textbf{0.914} & \textbf{0.821} & \textbf{0.862} & \textbf{0.812} & \textbf{0.801} & \textbf{0.799} & \underline{0.853} & \textbf{0.888} & \textbf{0.890} & \textbf{0.875} & \textbf{0.871} \\
$\ours{}$    & \textbf{0.600} & \underline{0.884} & \textbf{0.821} & \underline{0.787} & \underline{0.782} & \underline{0.765} & \underline{0.773} & \textbf{0.871} & 0.771 & 0.831 & 0.823 & \underline{0.821} \\
\bottomrule
\end{tabular}%
}
\vspace{-0.6em}
\end{table}

\textbf{Baselines.} We compare \ours{} against two families of state-of-the-art methods: (1) \textit{MoE-based fusion models} operating on WSI, text, and graph modalities, including Early Fusion, MoEPP, Switch Gate, and PathMoE~\cite{yu2026pathmoe} (\ItwoMoE); and (2) \textit{Concept-aware models}, including CBM~\cite{koh2020concept} and CEM~\cite{espinosa2022concept}, augmented with different backbones and concept hierarchies.

Detailed dataset preprocessing, optimization, and hardware configurations are deferred to Appendix~\ref{app:impl_details}.

\vspace{-0.4em}
\subsection{Main Results}
\label{sec:main_results}
\vspace{-0.4em}

Table~\ref{tab:per_class_metrics_main} compares standard multimodal fusion baselines: Early Fusion, MoEPP, Switch Gate, and PathMoE~\citep{yu2026pathmoe}, against \ours{}. Table~\ref{tab:per_class_metrics_concept} ablates the bottleneck design within the \ours{} family. The \emph{flat} variant uses a single concept layer of biomarkers; the \emph{hier} variant uses two layers, with morphology concepts in the first layer feeding into biomarkers in the second. The \emph{(CBM)} ablation replaces the continuous concept embeddings (CEM) with discrete CBM-style scalar predictions, isolating the contribution of the embedding-based bottleneck. The \emph{(WSI-only)} ablation removes the graph modality entirely, retaining only the WSI expert to quantify the value of multimodal evidence. All variants are evaluated under identical training settings.

\begin{table}[htbp]
\centering
\setlength{\tabcolsep}{4pt} 
\captionsetup{font=small,skip=2pt}
\caption{Concept-bottleneck variants of \ours{} on PBT. ``HG/LG''= HG/LG-CNS. ``NGT" = Non-glial tumor.}\label{tab:per_class_metrics_concept}
\vspace{0.2em}
\resizebox{0.9\linewidth}{!}{
\begin{tabular}{l cccc ccc c}
\toprule
& \multicolumn{4}{c}{Per-class F1} & \multicolumn{3}{c}{Macro} & Concept \\
\cmidrule(lr){2-5} \cmidrule(lr){6-8}
\begin{tabular}[b]{@{}l@{}}\ours{} \\ Variant\end{tabular} & Epend. & LG & HG & NGT & P & R & F1 & AUROC \\ 
\midrule
flat (CBM+bio)        & 0.400          & 0.389          & 0.591          & 0.278          & 0.421          & 0.426          & 0.415          & 0.645 \\
flat (morph)          & \textbf{0.600} & \textbf{0.884} & \textbf{0.821} & 0.787          & \textbf{0.782} & \textbf{0.765} & \textbf{0.773} & 0.672 \\
flat (bio)            & 0.500          & 0.878          & 0.809          & 0.810          & 0.768          & 0.753          & 0.749          & 0.726 \\
\midrule
hier (WSI+morph+bio)  & 0.500          & 0.869          & 0.812          & 0.772          & 0.752          & 0.741          & 0.738          & 0.700 \\
hier (CBM+morph+bio)  & 0.233          & 0.421          & 0.586          & 0.232          & 0.372          & 0.395          & 0.368          & 0.585 \\
hier (morph+bio)      & 0.500          & 0.873          & 0.798          & \textbf{0.818} & 0.760          & 0.740          & 0.747          & \textbf{0.761} \\
\bottomrule
\end{tabular}%
}
\end{table}


\vspace{-0.4em}
\subsection{Interpretability Analysis}
\label{sec:interpretability}
\vspace{-0.4em}
\ours{}'s gating network produces a per-sample weight vector
$\boldsymbol{\alpha} = \mathrm{softmax}(g(\mathbf{x}))$ over experts;
class-averaged $\boldsymbol{\alpha}$ reveals which evidence \ours{} relies
on for which tumor type, a quantitative view of modality usage opaque in
standard fusion models that lets us check whether learned specializations
align with histopathological expectations.

\begin{figure}[htbp]
\centering
\begin{subfigure}[t]{0.32\textwidth}
  \centering
  \includegraphics[width=\linewidth]{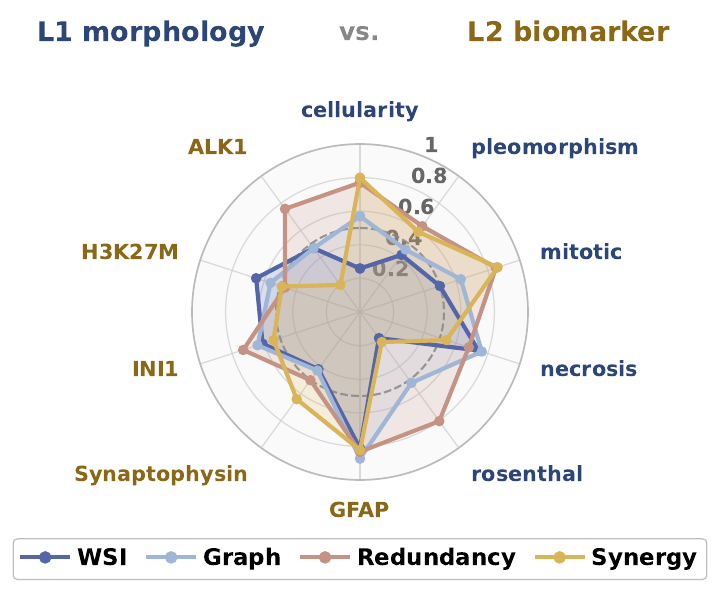}
  \caption{Per-expert concept AUROC.}
  \label{fig:per_expert_auroc}
\end{subfigure}\hfill
\begin{subfigure}[t]{0.32\textwidth}
  \centering
  \includegraphics[width=\linewidth]{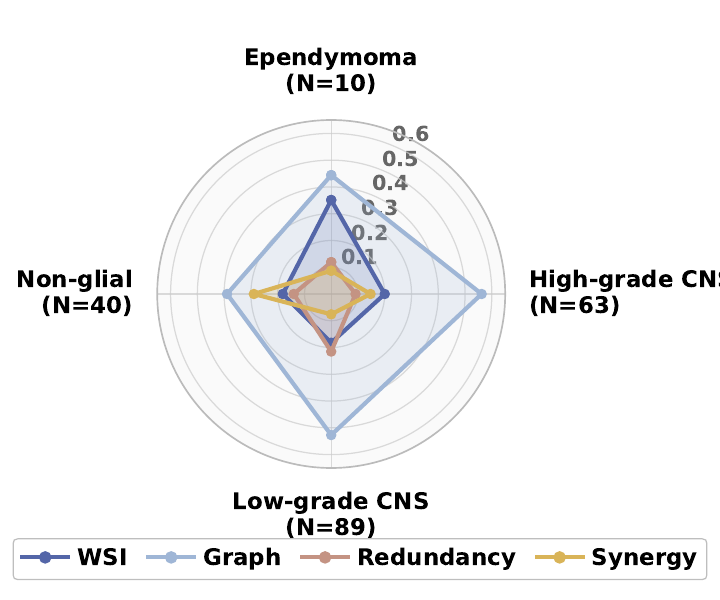}
  \caption{Per-class gate weights.}
  \label{fig:per_class_gate}
\end{subfigure}\hfill
\begin{subfigure}[t]{0.32\textwidth}
  \centering
  \includegraphics[width=\linewidth]{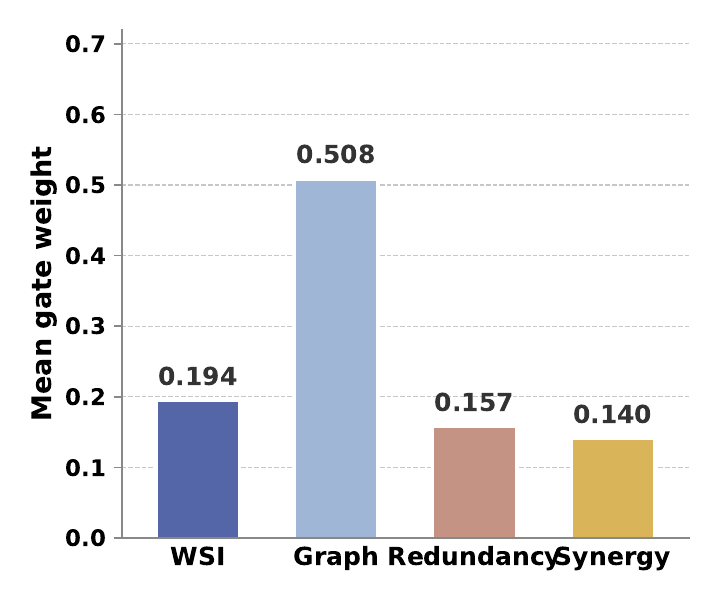}
  \caption{Overall routing distribution.}
  \label{fig:routing_dist}
\end{subfigure}

\vspace{0.5\baselineskip}

\begin{subfigure}[t]{0.49\textwidth}
  \centering
  \includegraphics[width=\linewidth]{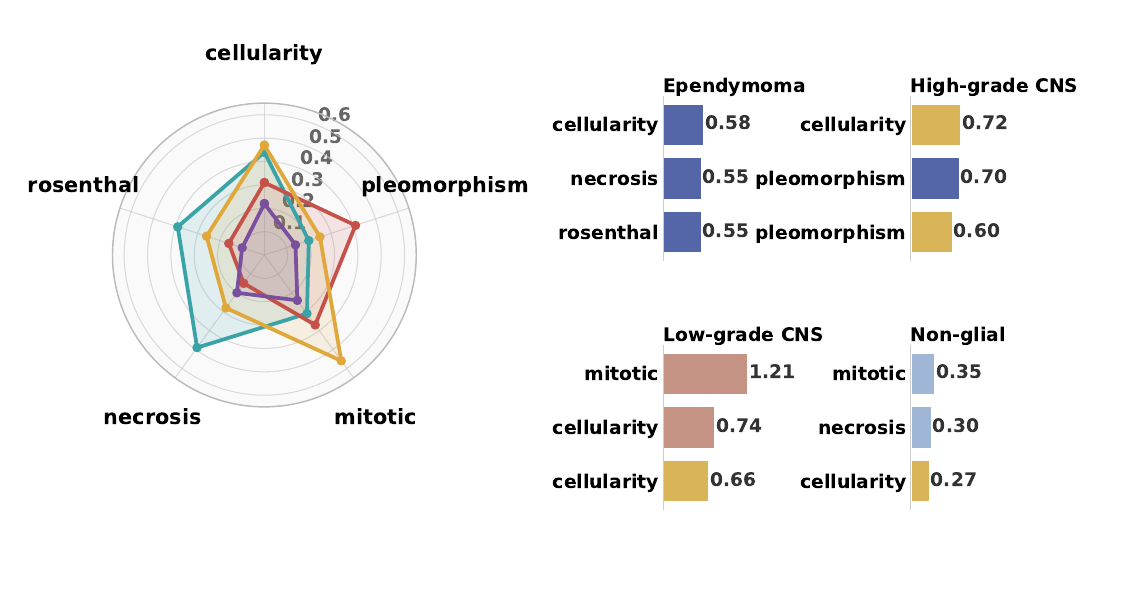}
  \caption{L1 morphology $\to$ stage-1 logit attribution.}
  \label{fig:l1_attr}
\end{subfigure}\hfill
\begin{subfigure}[t]{0.49\textwidth}
  \centering
  \includegraphics[width=\linewidth]{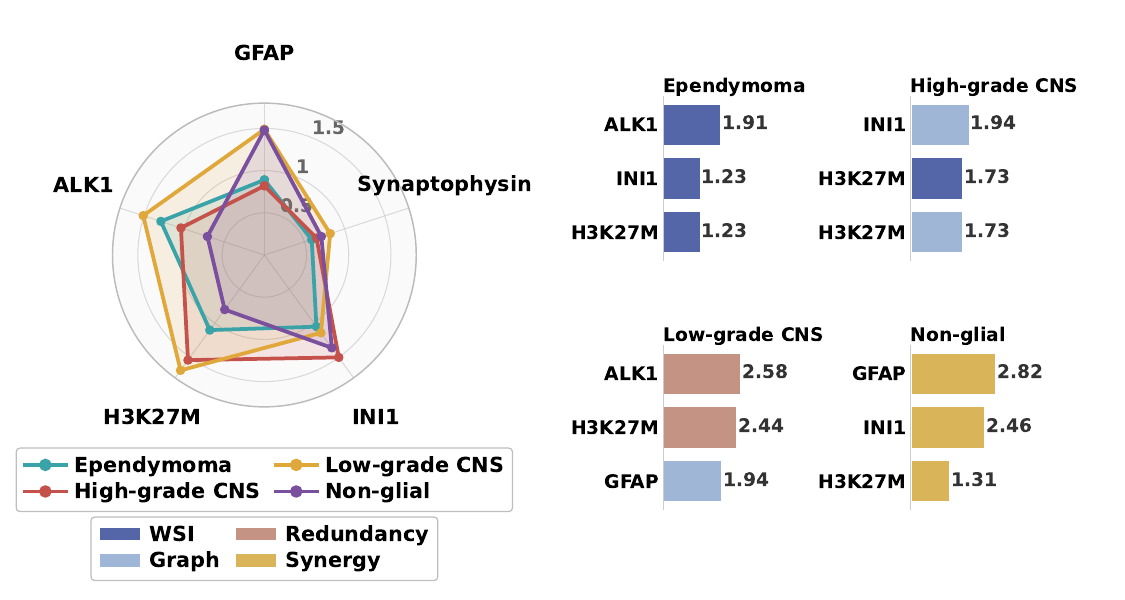}
  \caption{L2 biomarker $\to$ final logit attribution.}
  \label{fig:l2_attr}
\end{subfigure}

\caption{\textbf{Interpretability of \ours{} on PBT.}
\textbf{(a)} Per-expert concept AUROC (L1 morphology dark blue, L2 biomarker dark gold; dashed circle: chance).
\textbf{(b)} Per-class mean gate weights over the four experts (WSI, Graph, Redundancy, Synergy); $N$ next to class names.
\textbf{(c)} Overall routing distribution, marginalized over classes.
\textbf{(d, e)} Concept-to-logit attribution for L1 and L2. Spider plots: gate-aggregated contribution per class on each concept axis. Bar panels: top-3 (concept, expert) drivers per class, colored by dominant expert. Legend in (e) applies to both.}
\label{fig:interpretability}
\vspace{-0.6em}
\end{figure}

\textbf{Per-class expert specialization.} Routing in
Fig.~\ref{fig:per_class_gate} and \ref{fig:routing_dist} aligns with histopathological intuition:
\emph{Graph} dominates both glial classes (HG: 0.562, LG: 0.527), capturing
architectural cues at the cell-graph level; \emph{WSI} peaks for Ependymoma
(0.351), whose cytology is best resolved at tile resolution; \emph{Synergy}
specializes to Non-glial (0.289 vs.\ $0.08$--$0.15$ elsewhere, 2--3$\times$),
where tile cytology and graph organization provide complementary evidence;
and \emph{Redundancy} is most active for Low-grade CNS (0.215, $\sim$2$\times$
HG and Non-glial), matching the stereotyped histology of pilocytic
astrocytoma. Histological correlates per expert are in
Appendix~\ref{app:expert_specialization}.

\textbf{Hierarchical concept attribution.} Gradient$\times$input on the
stage-1 (L1) and final (L2) heads attributes each class logit to its driving
concepts; bar panels in Fig.~\ref{fig:l1_attr}, Fig.~\ref{fig:l2_attr} report the top-3
(concept, expert) pairs per class, colored by dominant expert. \emph{L1
morphology} recovers textbook cytology: pleomorphism and cellularity drive
HG (cellularity 0.72 Synergy; pleomorphism 0.70 WSI / 0.60 Synergy), the
anaplastic phenotype; mitotic and cellularity drive LG via Redundancy
(mitotic 1.21; cellularity 0.74 Redundancy / 0.66 Synergy), matching
pilocytic histology; Ependymoma routes cellularity, necrosis, and rosenthal
through WSI alone (0.58, 0.55, 0.55); and Non-glial L1 contributions are
uniformly small and Graph-routed (mitotic 0.35, necrosis 0.30, cellularity
0.27), indicating reliance on L2 over L1. \emph{L2 biomarkers} recover
molecular signatures: H3K27M dominates HG via \emph{both} modalities (INI1
1.94 Graph; H3K27M 1.73 WSI / 1.73 Graph), the DMG signature; Synergy
integrates GFAP-negative / INI1-loss evidence for Non-glial (GFAP 2.82, INI1
2.46, H3K27M 1.31), the ATRT workup; Redundancy pairs ALK1 (2.58) and H3K27M
(2.44) for LG alongside Graph-driven GFAP (1.94), using \emph{absence} of
high-grade markers as confirmatory evidence; and INI1 is used
\emph{bidirectionally}, supporting HG when retained (1.94) and Non-glial when
lost (2.46). \ours{} also learns negative evidence: Rosenthal fibers,
pathognomonic for pilocytic astrocytoma, receive the lowest L1 weights where
they should not apply (HG: 0.16; Non-glial: 0.10), with pleomorphism and
mitotic dominating the L1$\rightarrow$L2 path to H3K27M. Full attribution
maps and L1$\rightarrow$L2 details are in
Appendix~\ref{app:concept_attribution}.

\textbf{Inference-time reasoning trace.} Beyond what \ours{} learns during
training, at deployment it produces a per-slide three-part reasoning trace:
\textbf{(i)} the gate weights $\boldsymbol{\alpha}$ indicate which modalities
and interactions \ours{} uses; \textbf{(ii)} per-expert concept activations
expose the morphology and biomarker evidence behind the prediction;
\textbf{(iii)} comparing concept scores across experts diagnoses the
\emph{informational origin} of the decision, whether driven by shared
cross-modal consensus (Redundancy) or emergent multimodal interactions
(Synergy). Together these compose a clinically-readable justification for
the predicted class. Fig.~\ref{fig:per_expert_auroc} reveals a strikingly
non-uniform pattern: cellularity, mitotic, and ALK1 are reliably recovered
by Redundancy and Synergy experts (AUROC $0.76$--$0.86$) while unimodal WSI
and Graph experts sit near or below chance on the same concepts; conversely,
GFAP is recovered at AUROC $\geq 0.81$ in \emph{all four} experts, as
expected for a near-universal glial marker present in every modality view.
We attribute this non-uniformity to three factors (concept learning need not
be consistent across experts; high gate weight does not imply broad concept
coverage; and \ours{} routes supervision toward the cleaner cross-modal
signal); details are in Appendix~\ref{app:concept_auroc_analysis}.

\begin{figure}[ht]
\centering
\begin{minipage}[c]{0.135\textwidth}
\centering
\includegraphics[width=\textwidth]{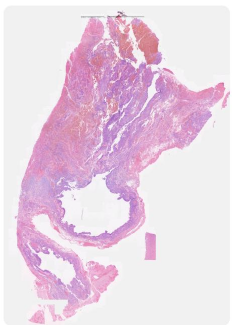}
\end{minipage}\hfill
\begin{minipage}[c]{0.865\textwidth}
\centering
\renewcommand{\arraystretch}{0.9}
\setlength{\tabcolsep}{2pt}
\fontsize{6.5pt}{7.8pt}\selectfont
\begin{tabular}{l c c p{6.4cm}}
\toprule
\multicolumn{4}{c}{\textbf{True: High-grade CNS \quad Predicted: High-grade CNS \checkmark}} \\
\midrule
\textbf{Concept (top-6)} & \textbf{Expected} & \textbf{Model} & \textbf{Clinical Rationale} \\
\midrule
pleomorphism  & high     & high (0.995)     & \multirow{6}{6.4cm}{\raggedright\textit{All six concepts inform H\&E assignment within high-grade CNS. INI1 retention and H3K27M negativity exclude two specific diagnoses; the remaining features keep this case in the high-grade CNS group.}} \\
mitotic       & high     & low (0.019)      & \\
INI1          & retained & retained (0.999) & \\
H3K27M        & --       & negative (0.000) & \\
cellularity   & high     & high (1.000)     & \\
necrosis      & high     & high (0.880)     & \\
\bottomrule
\end{tabular}
\end{minipage}
 \caption{\textbf{Reasoning trace for 
a high-grade CNS Example.}
Top-6 concepts driving \ours{}'s prediction with neuropathologist-expected status; 5/6 match.}
\label{fig:case_study}
\vspace{-0.6em}
\end{figure}

\textbf{Case study.} To evaluate the extent to which the concepts leveraged by \ours{} are consistent with neuropathological reasoning, we compare the six highest-contribution concepts for a correctly classified high-grade CNS slide with the diagnostic rationale provided by an independent, board-certified neuropathologist. (Fig.~\ref{fig:case_study}). Five of six (pleomorphism, INI1 retention, H3K27M-negative status, cellularity,
necrosis) match the pathologist's cited evidence in both identity and status;
Mitotic activity is the notable divergence (the report records six mitoses
per 10 high-power fields, yet the model assigned it low probability), but
the remaining evidence is sufficient to recover the correct class. The full
reasoning trace, gate weights selecting experts, morphology concepts
supplying cytologic evidence, biomarkers refining molecular subtype, is thus
inspectable in standard pathology vocabulary. Additional cases per class with
independent rationales are in Appendix~\ref{app:case_studies}.

\vspace{-0.4em}
\subsection{Training-Set Subsampling Analysis}
\label{sec:sample_efficiency}
\vspace{-0.4em}
To examine how the models behave when labeled training data are limited, we
repeat training under progressively smaller training sets while keeping the
same split protocol and hyperparameters. As shown in Fig.~\ref{fig:sample_efficiency},
\ours{} (full) achieves higher mean macro-F1 than PathMoE at the three reduced
training sizes. At $N=50$, the full model improves macro-F1 from $56.41$ to
$66.70$, with a smaller cross-split standard deviation ($11.73$ vs. $21.99$).
The advantage remains at $N=100$ ($80.02$ vs. $77.32$) and becomes more
pronounced at $N=150$ ($84.58$ vs. $72.70$). At the full training size
($N=164$), the gap narrows, and the bio-only variant obtains the highest mean
macro-F1 ($81.68$), with all three methods within a relatively small range.
These results suggest that concept supervision provides a useful inductive bias
in data-limited regimes. By encouraging the representation to align with
intermediate clinical concepts, the model may reduce reliance on purely
end-to-end label fitting when labels are scarce. This observation is consistent
with the concept-aligned localization view in
Theorem~\ref{thm:concept_localization}, while the crossover at the full training
size also indicates that the benefit of concept regularization is
data- and configuration-dependent rather than uniformly dominant.
The convergence curves in Fig.~\ref{fig:loss_curves} provide a complementary
optimization perspective. Under the same training configuration, both
concept-supervised variants reach low training cross-entropy earlier than
PathMoE. This trend is consistent with
Proposition~\ref{prop:aux_gradient_contraction}: when auxiliary concept losses
provide gradients aligned with the task objective, they can accelerate task-loss
contraction. We therefore interpret the subsampling and convergence results as evidence that hierarchical concept supervision can improve
limited-data behavior and optimization dynamics, while maintaining comparable performance.
\begin{figure}[htbp]
\centering
\begin{subfigure}[t]{0.245\textwidth}
  \centering
  \includegraphics[width=\linewidth]{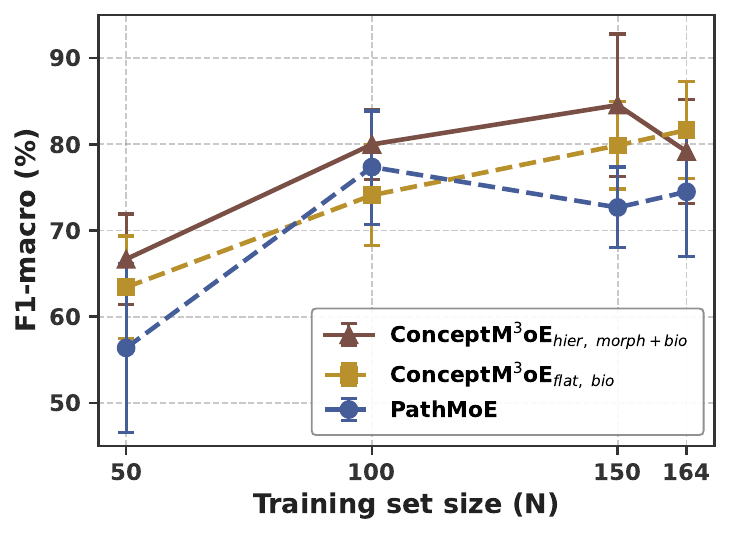}
  \caption{Sample efficiency.}
  \label{fig:sample_efficiency}
\end{subfigure}
\hfill
\begin{subfigure}[t]{0.245\textwidth}
  \centering
  \includegraphics[width=\linewidth]{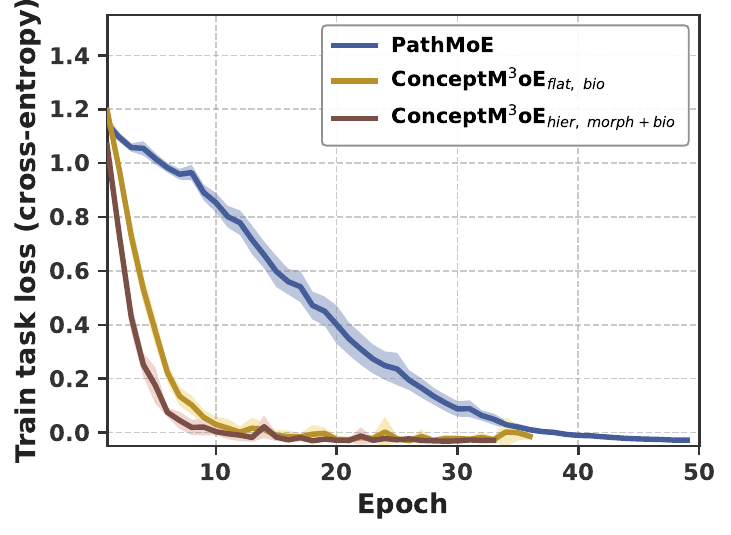}
  \caption{Train Loss Curve}
  \label{fig:loss_curves}
\end{subfigure}
\hfill
\begin{subfigure}[t]{0.245\textwidth}
  \centering
  \includegraphics[width=\linewidth]{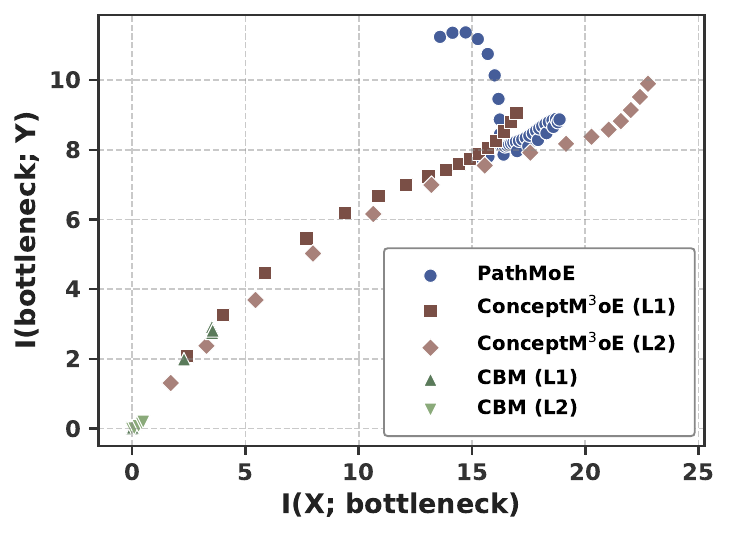}
  \caption{MI plane.}
  \label{fig:mi_plane}
\end{subfigure}
\hfill
\begin{subfigure}[t]{0.245\textwidth}
  \centering
  \includegraphics[width=\linewidth]{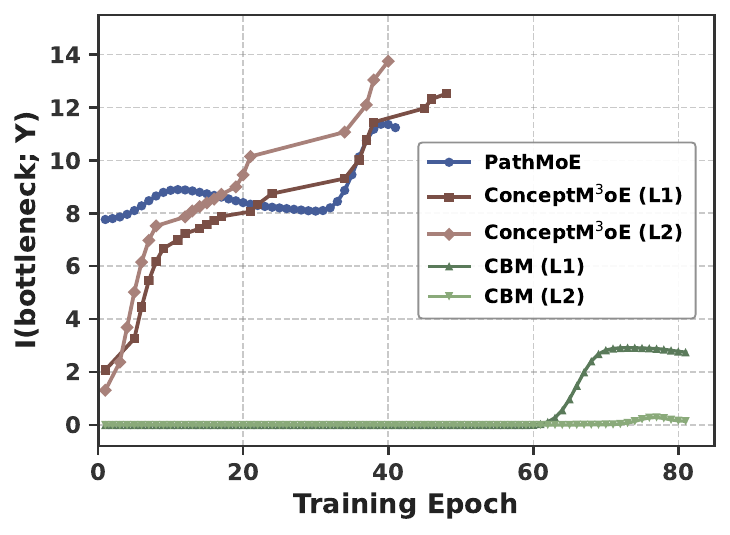}
  \caption{$I(C;Y)$ over epochs.}
  \label{fig:mi_curve}
\end{subfigure}
\caption{\textbf{Training analysis of \ours{} on PBT.}
\textbf{(a)} Macro-F1 vs.\ training-set size $N$, mean~$\pm$~SEM (Standard Error of the Mean) over 5 splits.
\textbf{(b)} Task cross-entropy loss at $N{=}164$, mean~$\pm$~std over 5 splits,
to epoch~50.
\textbf{(c)} Mutual information of bottleneck $C$ w.r.t.\ input $X$ and label
$Y$; marker size grows with epoch.
\textbf{(d)} Task information $I(C;Y)$ over epochs; CBM variants stay near $0$
(no class-relevant concept info), whereas \ours{} variants stay close to PathMoE.
In (c) and (d), PathMoE is the no-bottleneck baseline; \ours{}
and the CBM-\ours{} variant use (\textit{hier, morph+bio}) with supervision at L1
(morphology) or L2 (biomarker).}
\label{fig:training_and_info}
\vspace{-0.6em}
\end{figure}

\vspace{-0.4em}
\subsection{Mutual Information Analysis}
\label{sec:mi_analysis}
\vspace{-0.4em}
To explain why \ours{} matches PathMoE despite the concept bottleneck, we analyze them via the Information Plane framework~\cite{tishby2000information}, plotting bottleneck capacity $I(X;C)$ against task-relevant mutual information $I(C;Y)$ (details in Appendix~\ref{app:mi_details}). We compare three variants under identical 10-fold splits: (1) PathMoE latent features (no bottleneck, $d{=}768$); (2) our continuous \ours{}$_{\mathrm{hier,\,morph+bio}}$, tracking both its L1 (morphology, $d{=}80$) and L2 (biomarker, $d{=}80$) bottlenecks; and (3) the discrete CBM-\ours{} baseline. Fig.~\ref{fig:mi_plane} shows that continuous \ours{} bottlenecks begin at the origin and migrate toward PathMoE's upper-right region during training. At convergence, \ours{}$_{\mathrm{hier,\,morph+bio}}$ (L1) preserves task information $I(\hat{\mathbf{c}}_1;Y)$ comparable to PathMoE despite a ${\sim}10{\times}$ dimensionality reduction, while L2 slightly exceeds L1. In stark contrast, purely scalar CBM-\ours{} variants fail to preserve class-relevant information, lingering near zero throughout training. 

Along the training trajectory, \ours{}$_{\mathrm{hier,\,morph+bio}}$ (L2) accumulates higher capacity $I(X;C)$ than L1. This stems from our hierarchical architecture: L2 receives both raw encoder features and the L1 concept residual, expanding its effective information content rather than strictly constraining it. Corroborating this, Fig.~\ref{fig:mi_curve} shows $I(C;Y)$ rising in close correspondence with validation macro-F1, indicating MI consistently captures task-relevant structure. These findings offer an information-theoretic explanation for the performance parity observed in~\S\ref{sec:main_results}. When realized as learnable embeddings with interaction-aware routing, the concept bottleneck imposes negligible information loss relative to an unconstrained latent space. Its primary effect is to reorganize information into semantically interpretable axes rather than to restrict capacity.
\vspace{-0.5em}

%% file: conclusion.tex
\section{Discussion and Conclusion}
\label{sec:conclusion}
\vspace{-0.5em}

\ours{} shows that concept supervision can serve as a structural prior for multimodal pathology without sacrificing predictive performance. On pediatric brain tumor classification, the model achieves performance on par with strong multimodal baselines while producing concept-wise evidence profiles that align with neuropathological reasoning, including class-specific reliance on modality-specific and cross-modal experts. The sample-efficiency results further suggest that morphology and biomarker concepts provide useful inductive bias under limited data, and the mutual information analysis indicates that concept embeddings preserve task-relevant signal while reorganizing it into interpretable axes. These findings support concept-centered expert modeling as a practical direction for reliable computational pathology, where predictions can be assessed through both accuracy and the clinical evidence structure behind them. Future work should validate the framework across larger multi-institutional cohorts and study how to specify concept sets with appropriate coverage, granularity, and clinical relevance.

%% file: checklist.tex
\section*{NeurIPS Paper Checklist}

\begin{enumerate}

\item {\bf Claims}
    \item[] Question: Do the main claims made in the abstract and introduction accurately reflect the paper's contributions and scope?
    \item[] Answer: \answerYes{} 
    \item[] Justification: The abstract and introduction explicitly state the contributions (interpretable MoE, hierarchical concept bottleneck, sample efficiency), which are directly supported by the experimental results in Section \ref{sec:experiments}.

\item {\bf Limitations}
    \item[] Question: Does the paper discuss the limitations of the work performed by the authors?
    \item[] Answer: \answerYes{} 
    \item[] Justification: Yes, limitations such as the need for larger multi-institutional cohorts, appropriate concept set specification, and risks of automation bias are discussed in Section \ref{sec:conclusion} (Discussion and Conclusion) and Appendix \ref{app:broader_impacts} (Broader Impacts).

\item {\bf Theory assumptions and proofs}
    \item[] Question: For each theoretical result, does the paper provide the full set of assumptions and a complete (and correct) proof?
    \item[] Answer: \answerYes{} 
    \item[] Justification: Yes, assumptions for the concept-aligned localization theorem and the auxiliary-gradient contraction proposition are stated in Section \ref{sec:training_objective}, with full proofs provided in Appendices \ref{app:proof_soft_concept_augmentation}, \ref{app:proof_concept_localization}, and \ref{app:proof_aux_gradient_contraction}.

    \item {\bf Experimental result reproducibility}
    \item[] Question: Does the paper fully disclose all the information needed to reproduce the main experimental results of the paper to the extent that it affects the main claims and/or conclusions of the paper (regardless of whether the code and data are provided or not)?
    \item[] Answer: \answerYes{} 
    \item[] Justification: Yes, the experimental setup, baselines, and evaluation protocol are detailed in Section \ref{sec:experiments}, with comprehensive hyperparameters, network dimensions, and training configurations provided in Appendix \ref{app:impl_details}.

\item {\bf Open access to data and code}
    \item[] Question: Does the paper provide open access to the data and code, with sufficient instructions to faithfully reproduce the main experimental results, as described in supplemental material?
    \item[] Answer: \answerNo{} 
    \item[] Justification: The institutional pediatric brain tumor (PBT) dataset cannot be publicly released due to patient privacy and Institutional Review Board (IRB) restrictions. However, we evaluated on the public TCGA dataset and will provide training code upon acceptance.

\item {\bf Experimental setting/details}
    \item[] Question: Does the paper specify all the training and test details (e.g., data splits, hyperparameters, how they were chosen, type of optimizer) necessary to understand the results?
    \item[] Answer: \answerYes{} 
    \item[] Justification: Yes, training details including optimizer, learning rate, batch size, epochs, early stopping criteria and specific hyperparameters for both datasets are described in Appendix \ref{app:impl_details}.

\item {\bf Experiment statistical significance}
    \item[] Question: Does the paper report error bars suitably and correctly defined or other appropriate information about the statistical significance of the experiments?
    \item[] Answer: \answerYes{} 
    \item[] Justification: Yes, standard deviations and standard errors over cross-validation splits are reported for the architectural ablations (Table \ref{tab:ablation}) and training-set subsampling analysis (Figure \ref{fig:sample_efficiency}) to support the main claims.

\item {\bf Experiments compute resources}
    \item[] Question: For each experiment, does the paper provide sufficient information on the computer resources (type of compute workers, memory, time of execution) needed to reproduce the experiments?
    \item[] Answer: \answerYes{} 
    \item[] Justification: Yes, the compute infrastructure (a single NVIDIA RTX A6000 48 GB GPU) is explicitly documented in Appendix \ref{app:impl_details}.
    
\item {\bf Code of ethics}
    \item[] Question: Does the research conducted in the paper conform, in every respect, with the NeurIPS Code of Ethics \url{https://neurips.cc/public/EthicsGuidelines}?
    \item[] Answer: \answerYes{} 
    \item[] Justification: Yes, all research conducted in this paper conforms with the NeurIPS Code of Ethics.

\item {\bf Broader impacts}
    \item[] Question: Does the paper discuss both potential positive societal impacts and negative societal impacts of the work performed?
    \item[] Answer: \answerYes{} 
    \item[] Justification: Yes, a dedicated "Broader Impacts" section is provided in Appendix \ref{app:broader_impacts}, discussing both positive implications (transparency in clinical AI) and potential negative risks (automation bias, representation bias) along with mitigations.
    
\item {\bf Safeguards}
    \item[] Question: Does the paper describe safeguards that have been put in place for responsible release of data or models that have a high risk for misuse (e.g., pre-trained language models, image generators, or scraped datasets)?
    \item[] Answer: \answerNA{} 
    \item[] Justification: The predictive models developed in this work are highly specific to specialized histopathology classification and pose no high risk for general public misuse.

\item {\bf Licenses for existing assets}
    \item[] Question: Are the creators or original owners of assets (e.g., code, data, models), used in the paper, properly credited and are the license and terms of use explicitly mentioned and properly respected?
    \item[] Answer: \answerYes{} 
    \item[] Justification: Yes, existing assets such as the TCGA dataset, UNIv2 features, and base architectures are properly cited in the main text and Appendix \ref{app:impl_details}.

\item {\bf New assets}
    \item[] Question: Are new assets introduced in the paper well documented and is the documentation provided alongside the assets?
    \item[] Answer: \answerNA{} 
    \item[] Justification: No new publicly releasable datasets or standalone foundation models are introduced as primary assets in this submission due to institutional privacy constraints.

\item {\bf Crowdsourcing and research with human subjects}
    \item[] Question: For crowdsourcing experiments and research with human subjects, does the paper include the full text of instructions given to participants and screenshots, if applicable, as well as details about compensation (if any)? 
    \item[] Answer: \answerNA{} 
    \item[] Justification: The research relies exclusively on retrospective, de-identified clinical data and does not involve crowdsourcing or prospective human subjects experimentation.

\item {\bf Institutional review board (IRB) approvals or equivalent for research with human subjects}
    \item[] Question: Does the paper describe potential risks incurred by study participants, whether such risks were disclosed to the subjects, and whether Institutional Review Board (IRB) approvals (or an equivalent approval/review based on the requirements of your country or institution) were obtained?
    \item[] Answer: \answerYes{} 
    \item[] Justification: Yes, Appendix \ref{app:impl_details} explicitly states that the institutional retrospective PBT cohort was collected and utilized under an approved Institutional Review Board (IRB) protocol.

\item {\bf Declaration of LLM usage}
    \item[] Question: Does the paper describe the usage of LLMs if it is an important, original, or non-standard component of the core methods in this research? Note that if the LLM is used only for writing, editing, or formatting purposes and does \emph{not} impact the core methodology, scientific rigor, or originality of the research, declaration is not required.
    \item[] Answer: \answerYes{} 
    \item[] Justification: Yes, the use of GPT-4 for extracting morphology concept targets from pathology text reports is explicitly described, and the exact prompt used is provided in Appendix \ref{app:concept}.

\end{enumerate}

%% file: appendix.tex

\appendix

\section{Extended Experimental Setup and Implementation Details}
\label{app:impl_details}
\vspace{-0.4em}
\subsection{Dataset Details and Preprocessing}
\textbf{PBT Dataset.} The institutional cohort was collected under IRB approval from Dell Children’s
Medical Center. While the initial raw cohort comprised 253 diagnostic H\&E WSIs (age $\le 29$ years), we categorized 199 cases with sufficient multimodal alignment into the four primary diagnostic classes used in our experiments. Patient-level splits were strictly enforced. Unmentioned concepts in the pathology reports are masked during the loss computation.

\textbf{TCGA Dataset.} Diagnostic H\&E WSIs from TCGA-GBM and TCGA-LGG were accessed via the Genomic Data Commons (GDC) ~\cite{heath2021nci, cancer2008comprehensive, cancer2015comprehensive}, filtered for patients aged $\le 29$ years to align with the pediatric/young-adult scope of our primary cohort. We categorized 55 cases with sufficient multimodal alignment into two primary diagnostic classes used in our experiments. Patient-level splits were strictly enforced. Unmentioned concepts in the pathology reports are masked during the loss computation.
\vspace{-0.4em}
\subsection{Training and Optimization}
\vspace{-0.4em}
Models are trained end-to-end using the Adam optimizer. For the primary PBT cohort, we utilize a learning rate of $2 \times 10^{-4}$, a batch size of $16$, and a dropout rate of $p=0.1$ for regularization (see Table~\ref{tab:hparams} for the dataset-specific adjustments applied to the TCGA cohort). Across all experiments, we apply a cosine annealing learning rate schedule. Models are trained for up to 150 epochs with an early stopping patience of 30 epochs, monitored on the validation macro-F1 score. The checkpoint with the highest validation macro-F1 is selected for final testing. 

Each cross-validation fold uses identical pre-defined patient-level splits (consistent with the PathMoE protocol, yielding ${\sim}20$ test slides per fold) and is initialized with a fixed random seed to ensure exact reproducibility. All experiments were conducted on a single NVIDIA RTX A6000 (48\,GB) GPU.
\vspace{-0.4em}
\subsection{Hyperparameters}
\label{sec:hparams}
\vspace{-0.4em}
Table~\ref{tab:hparams} details the complete set of architectural and optimization hyperparameters used for \ours{} across both the PBT and TCGA evaluation cohorts. While the core multimodal mixture-of-experts architecture and bottleneck dimensions remain consistent across tasks, dataset-specific regularization strategies are applied to account for differences in cohort scale and complexity. Specifically, for the external TCGA validation dataset, we employ a lighter GNN backbone (2 layers vs. 3 layers), increased dropout ($0.6$ vs. $0.1$), a reduced batch size, and heavier L1 concept regularization ($\lambda_1=0.9$) without L2 hierarchy ($\lambda_2=0.0$) to mitigate potential overfitting.

\begin{table}[htbp]
\vspace{-0.8em}
\centering
\caption{\ours{} architectural and optimization hyperparameters. Values shared across both datasets are centered.}
\label{tab:hparams}
\small
\begin{tabular}{@{}lcc@{}}
\toprule
\textbf{Hyperparameter} & \textbf{PBT (4-class)} & \textbf{TCGA (2-class)} \\
\midrule
\multicolumn{3}{@{}l}{\textit{Architecture Details:}} \\
Hidden / token dim $d$ & \multicolumn{2}{c}{$256$} \\
Fusion tokens $P$ & \multicolumn{2}{c}{$16$} \\
Experts / routing mechanism & \multicolumn{2}{c}{$4$ / dense softmax} \\
Expert MLP $\cdot$ Fusion $\cdot$ Predictor layers & \multicolumn{2}{c}{$2 \cdot 1 \cdot 1$} \\
Concept embed dim $m$ & \multicolumn{2}{c}{$16$} \\
Layer-1 residual dim & \multicolumn{2}{c}{$16$} \\
GNN layers / hidden / dropout & $3$ / $256$ / $0.1$ & $2$ / $128$ / $0.5$ \\
\midrule
\multicolumn{3}{@{}l}{\textit{Optimization Details:}} \\
$\lambda_1,\,\lambda_2$ (Concept loss weights) & $0.5,\,0.3$ & $0.9,\,0.0$ \\
$\lambda_{\text{int}}$ (Interaction loss weight) & $0.1$ & $10^{-2}$ \\
Optimizer / learning rate & Adam / $2 \times 10^{-4}$ & Adam / $10^{-4}$ \\
Batch size / model dropout & $16$ / $0.1$ & $8$ / $0.6$ \\
Epochs / patience / schedule & \multicolumn{2}{c}{$150$ / $30$ / cosine annealing} \\
\bottomrule
\end{tabular}
\vspace{-0.8em}
\end{table}

\vspace{-0.4em}
\section{Detailed expert specialization analysis}
\label{app:expert_specialization}
\vspace{-0.4em}
We expand the histopathological correlates behind each expert's
specialization in Sec.~\ref{sec:interpretability}. The \emph{Graph}
expert's dominance on glial classes captures architectural features such as
perivascular pseudorosettes, microvascular proliferation, and infiltrative
growth, all well represented at the cell-graph level. The \emph{WSI}
expert's peak for Ependymoma reflects its distinctive cytology, pseudorosettes
with bland nuclei and canal-like structures, resolved at tile resolution.
The \emph{Synergy} specialization to Non-glial corresponds to ATRT and
embryonal tumors, where tile-level cytology and graph-level organization
provide complementary, non-redundant evidence that neither modality alone
resolves. The \emph{Redundancy} expert's preference for Low-grade CNS
matches the highly stereotyped histology of pilocytic astrocytoma, where
modalities tend to agree.
\vspace{-0.4em}
\section{Per-class concept attribution maps}
\label{app:concept_attribution}
\vspace{-0.4em}
We provide complementary detail to the per-class attribution maps in
Fig.~\ref{fig:l1_attr} and ~\ref{fig:l2_attr}, focusing on gate-aggregated values
and within-expert structure not covered by the top-3 bar panels in the
main text.

\textbf{L1 morphology.}
At the gate-aggregated level, pleomorphism is the top L1 contributor for
High-grade CNS (0.41), and is localized to the WSI expert (0.70) as a
per-cell cytologic feature best resolved at tile resolution. The Synergy
expert's two largest L1 weights, cellularity (0.72) and pleomorphism
(0.60), both for High-grade CNS, reproduce the ``more cells, abnormally
shaped'' combination pathologists use jointly to grade gliomas. The
Rosenthal-fiber down-weighting noted in the main text shows \ours{} has
learned \emph{where a concept should not be used}, an objective as
informative as positive evidence.

\textbf{L2 biomarkers.}
At the gate-aggregated level, H3K27M is the top L2 contributor for
High-grade CNS (1.54), recovered at similar magnitude by both WSI (1.73)
and Graph (1.73): both modalities access the diffuse-midline-glioma signal.
INI1 is second for High-grade CNS (1.50) and a top contributor for
Non-glial (1.36), the gate-aggregated counterpart to the per-expert
bidirectional pattern in the main text. Within the Redundancy expert, the
three largest L2 weights for Low-grade CNS, ALK1 (2.58), H3K27M (2.44),
and INI1 (1.93), are all high-grade-associated markers, reinforcing the
absence-as-evidence interpretation.

\textbf{L1$\rightarrow$L2 routing.}
Full L1$\rightarrow$L2 attribution maps and computational details for the
gate-aggregated and per-expert attribution scores are provided in
Appendix~\ref{app:attribution_computation}.

\vspace{-0.4em}
\section{Per-expert concept-AUROC analysis}
\label{app:concept_auroc_analysis}
\vspace{-0.4em}
Three observations explain the non-uniform per-expert AUROC pattern noted
in Sec.~\ref{sec:interpretability}.

\textbf{(1) Concept learning need not be consistent across experts.}
By construction, each \ours{} expert is trained on a perturbation that
isolates a different interaction type and is explicitly encouraged to
capture mutually exclusive information; a concept will be recovered most
strongly wherever its discriminative signal aligns best with that expert's
input view.

\textbf{(2) High gate weight does not imply broad concept coverage.}
The Graph expert receives the largest aggregate gate weight
(Fig.~\ref{fig:interpretability}(c)) yet learns relatively few concepts
at high AUROC, because classification and concept supervision are distinct
objectives: Graph acts as \ours{}'s diagnostic ``heavy lifter,'' while
Redundancy and Synergy function as specialized tools that activate only
when cross-modal evidence is needed.

\textbf{(3) Routing toward the cleaner signal.}
Morphology concepts like cellularity and mitotic activity are partially
visible in raw tiles, but \ours{} identifies that their cleaner signal is
fundamentally cross-modal and routes the supervision accordingly. GFAP is
the informative counter-example: as a near-universal glial marker, its
signal is present in every modality view, which is why all four experts
recover it.

\vspace{-0.4em}
\section{Concept attribution computation}
\label{app:attribution_computation}
\vspace{-0.4em}

For each test sample we compute the per-concept contribution to the class logit, both within each expert and after gating, then average within class and across folds. The full per-class L1 morphology and L2 biomarker attribution maps are shown as panels (d) and (e) of Fig.~\ref{fig:interpretability} in the main text.

We use \textbf{gradient $\times$ input} attribution to quantify each concept's contribution to a downstream output. For a scalar output $y$ and a $d$-dimensional concept embedding $\mathbf{c} \in \mathbb{R}^{d}$, the contribution of $\mathbf{c}$ to $y$ on a single sample is
\begin{equation}
\phi(\mathbf{c} \to y)
\;=\; \left| \sum_{i=1}^{d} \frac{\partial y}{\partial c_i} \, c_i \right|.
\label{eq:gradinput}
\end{equation}
This captures both how sensitive $y$ is to $\mathbf{c}$ (gradient) and what is actually present in the sample (input value); a concept the model weights heavily but whose embedding is near zero contributes little, and vice versa. Each L1 morphology and L2 biomarker concept is represented by a $16$-dimensional embedding, so the inner sum runs over $d{=}16$.

\textbf{Forward paths.}
We instantiate Eq.~\ref{eq:gradinput} for three forward paths in
\ours{}:
\begin{itemize}[leftmargin=1.2em,itemsep=0pt,topsep=2pt]
  \item \textbf{L1 $\to$ class.} The L1 bottleneck $\mathbf{b}_1 \in
        \mathbb{R}^{80}$ (concatenation of five L1 concept embeddings)
        is fed to the stage-1 head, producing class logits
        $\mathbf{s}^{(1)} \in \mathbb{R}^{4}$. We attribute the class-$c$
        logit $s^{(1)}_c$ to each L1 concept via Eq.~\ref{eq:gradinput}
        applied to the corresponding 16-d slice of $\mathbf{b}_1$.
  \item \textbf{L2 $\to$ class.} The final head receives
        $[\mathrm{softmax}(\mathbf{s}^{(1)}); \mathbf{b}_2]$, where
        $\mathbf{b}_2$ is the L2 bottleneck. We attribute the final
        class-$c$ logit to each L2 concept by gradients with respect to
        the appropriate slice of $\mathbf{b}_2$.
  \item \textbf{L1 $\to$ L2.} Each L2 concept embedding is produced by
        a block whose input concatenates the expert latent $\mathbf{z}$
        with $\mathbf{b}_1$, and combines a predictive scorer path with
        a residual term. We report attribution against the squared-norm
        of the L2 embedding $\|\mathbf{c}^{(2)}_{j}\|^{2}$, since this
        is the quantity that flows downstream.
\end{itemize}

For each path we obtain one attribution number per (sample, expert, source-concept, target). We aggregate as follows: (i)~restrict to test samples whose ground-truth label equals the target class~$c$; (ii)~average within each fold to obtain a (expert, source-concept) matrix; (iii)~report the across-fold mean and standard deviation, excluding fold/class pairs with no samples. Two views are produced from the same per-fold matrices: a \emph{per-expert} view that retains the expert dimension, and a \emph{gate-aggregated} view obtained by weighting each expert by its sample-level gate weight $\boldsymbol{\alpha}$ before averaging.

For L1$\to$L2 we differentiate the pre-sigmoid scorer score rather than the post-sigmoid probability, since the latter saturates for confidently predicted concepts and produces unstable gradients. Attribution magnitudes are not comparable \emph{across} the three paths because their target functions differ in scale; only within-row (which concept matters most for a given class) and within-column (which class a concept is most informative for) comparisons are meaningful. Finally, gradient $\times$ input is a local linearization at the operating point of the trained model and does not measure counterfactual concept removal, which would require intervention experiments left to future work.
\vspace{-0.4em}
\section{Concept Vocabulary and Extraction Pipeline}
\label{app:concept}
\vspace{-0.4em}
This appendix describes the two-level concept hierarchy used in \ours{}, the extraction protocol for obtaining concept targets from clinical data, and the encoding scheme used during training.
\vspace{-0.4em}
\subsection{Concept Vocabulary}
\vspace{-0.4em}
We define a two-level concept hierarchy. The first level captures histopathological findings described in pathology reports; the second level captures immunohistochemical (IHC) and molecular markers obtained from clinical assays.
\vspace{-0.4em}
\subsubsection{Level 1: Morphology Concepts ($K_1=5$)}
\vspace{-0.4em}
Each concept is scored by a large language model on an ordinal scale derived from the descriptive language used in pathology reports. The five concepts and their category sets are:\vspace{4pt}

\begin{center}
\begin{tabular}{@{}lll@{}}
\toprule
\textbf{Concept} & \textbf{Ordinal categories (low $\to$ high)} & \textbf{Target dim} \\
\midrule
Cellularity      & absent, low, mild\_moderate, high                         & 4 \\
Pleomorphism     & absent, mild, marked                                    & 3 \\
Mitotic activity & absent, rare, present\_not\_increased, increased, high  & 5 \\
Necrosis         & absent, focal, definite, pseudopalisading               & 4 \\
Rosenthal fibers & absent, focal, abundant                                 & 3 \\
\bottomrule
\end{tabular}
\end{center}

Each concept is represented as a one-hot vector over its categories during training, yielding $4+3+5+4+3 = 19$ binary targets. When a report does not mention a concept, the entry is marked \texttt{not\_mentioned} and all categories for that concept are masked in the loss.
\vspace{-0.4em}
\subsubsection{Level 2: Biomarker Concepts ($K_2=5$)}
\vspace{-0.4em}
These are binary (positive/negative) markers obtained directly from clinical IHC staining and, where available, DNA methylation profiling:

\begin{center}
\begin{tabular}{@{}lll@{}}
\toprule
\textbf{Concept} & \textbf{Source} & \textbf{Target} \\
\midrule
GFAP          & IHC & POS / NEG \\
Synaptophysin & IHC & POS / NEG \\
INI1          & IHC & POS / NEG \\
H3K27M        & IHC & POS / NEG \\
ALK1          & IHC & POS / NEG \\
\bottomrule
\end{tabular}
\end{center}

Each biomarker is a single binary target. When a specific IHC stain was not performed for a given case, that concept is masked in the loss.
\vspace{-0.4em}
\subsection{Extraction Pipeline}
\vspace{-0.4em}
\subsubsection{Level 1: Morphology Concepts (\texorpdfstring{$K_1=5$}{K1=5})}

Free-text pathology reports for all 199 cases in the PBT dataset were processed using GPT-4~\cite{achiam2023gpt}. The prompt instructed the model to read each report and assign an ordinal category (from the predefined vocabulary above) for each of the five L1 concepts. A structured JSON output schema was enforced to ensure consistent formatting. The prompt template is reproduced below:

\begin{quote}\small\ttfamily
You are a pathologist assistant. Read the following pathology report
and score each morphological feature on the scale provided. If a
feature is not mentioned, respond with "not\_mentioned".

Report: \{report\_text\}

Concepts:
- cellularity: absent / low / mild\_moderate / high
- pleomorphism: absent / mild / marked
- mitotic: absent / rare / present\_\allowbreak not\_\allowbreak increased / increased / high
- necrosis: absent / focal / definite / pseudopalisading
- rosenthal: absent / focal / abundant

Return a JSON object with one key per concept.
\end{quote}

The extracted labels were manually reviewed for a random subset of 20 reports. Inter-rater agreement with a board-certified pathologist exceeded 85\% for cellularity, necrosis, and vascular proliferation; pleomorphism and mitotic activity showed moderate agreement (70--75\%), consistent with the inherent subjectivity of these features in pediatric brain tumor histology.
\vspace{-0.4em}
\subsubsection{Level 2: Biomarker Concepts (\texorpdfstring{$K_2=5$}{K2=5})}
\vspace{-0.4em}
IHC staining results and methylation classifications were extracted from the structured fields of our institution's pathology database. These are categorical clinical variables (POS / NEG / not performed) requiring no language model extraction. Methylation-based tumor subtyping (e.g., ZFTA-RELA fused ependymoma) was recorded where available but not used as a direct concept target; it served only to validate diagnostic labels.
\vspace{-0.4em}
\subsection{Target Encoding and Loss Masking}
\vspace{-0.4em}
For training, L1 categorical labels are one-hot encoded per concept. For a concept with $V$ ordinal categories, the target $\mathbf{t}_{k} \in \{0,1\}^{V}$ is supervised with binary cross-entropy:

\[
\mathcal{L}_{\mathrm{concept}}^{\mathrm{L1}}
= \frac{1}{\sum_k m_k}
\sum_{k=1}^{K_1} m_k \cdot
\frac{1}{V_k} \sum_{v=1}^{V_k}
\mathrm{BCE}(p_{k,v}, t_{k,v}),
\]

where $m_k = 0$ if the concept was marked \texttt{not\_mentioned} and $m_k = 1$ otherwise. L2 biomarkers are single binary targets ($V_k=1$) supervised identically with masked BCE.
\vspace{-0.4em}
\section{Case Study}
\label{app:case_studies}
\vspace{-0.4em}
We extend the case study of Sec.~\ref{sec:interpretability} to the
remaining three tumor classes. For each class, we select one
correctly-classified slide and report the concepts contributing most to
\ours{}'s prediction, alongside their textbook-expected status, the
model's predicted status, and the diagnostic rationale provided by an
independent board-certified neuropathologist
(Figures~\ref{tab:case_lg}--\ref{tab:case_ng}). Across all three cases,
the pathologist's reasoning closely tracks the concepts the model
relied on, supporting the claim that \ours{} produces clinically
faithful, not merely accurate, predictions.

\begin{figure}[ht]
\centering
\begin{minipage}[c]{0.25\textwidth}
\centering
\includegraphics[width=\textwidth]{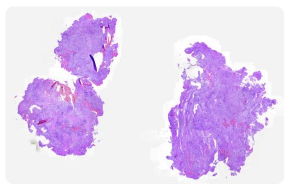}\\
{\scriptsize WSI\_000214}
\end{minipage}\hfill
\begin{minipage}[c]{0.73\textwidth}
\centering
\renewcommand{\arraystretch}{1.25}
\setlength{\tabcolsep}{4pt}
\footnotesize
\begin{tabular}{l|c|c|p{4.8cm}}
\toprule
\multicolumn{4}{c}{\textbf{True: Low-grade CNS \quad Predicted: Low-grade CNS \checkmark}} \\
\midrule
\textbf{Concept} & \textbf{Expected} & \textbf{Model} & \textbf{Clinical Rationale} \\
\midrule
Rosenthal     & high & high (0.611) & \multirow{6}{4.8cm}{\raggedright \textit{All six concepts align with the expert feature set for assigning low-grade CNS. GFAP positivity, low cellularity, low likelihood of necrosis, and low pleomorphism are classic low-grade features. Rosenthal fibers are non-specific but helpful within the low-grade category.}} \\
GFAP          & high & high (0.997) & \\
cellularity   & low  & low  (0.426) & \\
mitotic       & low  & low  (0.127) & \\
necrosis      & low  & low  (0.069) & \\
pleomorphism  & low  & low  (0.363) & \\
\bottomrule
\end{tabular}
\end{minipage}
\caption{\textbf{Low-grade CNS reasoning trace}.
Clinical rationale provided by an independent neuropathologist.}
\label{tab:case_lg}
\end{figure}

\begin{figure}[ht]
\centering
\begin{minipage}[c]{0.25\textwidth}
\centering
\includegraphics[width=\textwidth]{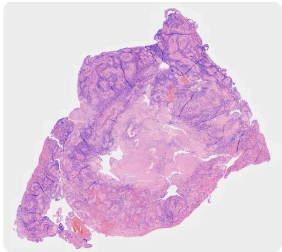}\\
{\scriptsize WSI\_000009}
\end{minipage}\hfill
\begin{minipage}[c]{0.73\textwidth}
\centering
\renewcommand{\arraystretch}{1.25}
\setlength{\tabcolsep}{4pt}
\footnotesize
\begin{tabular}{l|c|c|p{4.8cm}}
\toprule
\multicolumn{4}{c}{\textbf{True: Ependymoma \quad Predicted: Ependymoma \checkmark}} \\
\midrule
\textbf{Concept} & \textbf{Expected} & \textbf{Model} & \textbf{Clinical Rationale} \\
\midrule
Synaptophysin & high & low  (0.023) & \multirow{6}{4.8cm}{\raggedright \textit{All six concepts are reasonable, but from a human perspective do not, on their own, justify assignment to the ependymoma class. This class is largely decided by experts based on tumor morphology, growth pattern, GFAP positivity, retained INI1 protein, and cellular density.}} \\
INI1          & high & high (0.988) & \\
GFAP          & high & high (0.996) & \\
H3K27M        & low  & low  (0.013) & \\
ALK1          & low  & low  (0.077) & \\
Rosenthal     & low  & high (0.576) & \\
\bottomrule
\end{tabular}
\end{minipage}
\caption{\textbf{Ependymoma reasoning trace}.
Clinical rationale provided by an independent neuropathologist.}
\label{tab:case_epn}
\end{figure}

\begin{figure}[ht]
\centering
\begin{minipage}[c]{0.25\textwidth}
\centering
\includegraphics[width=\textwidth]{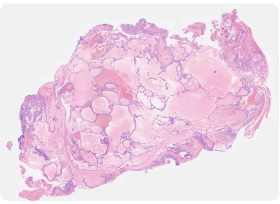}\\
{\scriptsize WSI\_000076}
\end{minipage}\hfill
\begin{minipage}[c]{0.73\textwidth}
\centering
\renewcommand{\arraystretch}{1.7}
\setlength{\tabcolsep}{4pt}
\footnotesize
\begin{tabular}{l|c|c|p{5.0cm}}
\toprule
\multicolumn{4}{c}{\textbf{True: Non-glial \quad Predicted: Non-glial \checkmark}} \\
\midrule
\textbf{Concept} & \textbf{Expected} & \textbf{Model} & \textbf{Clinical Rationale} \\
\midrule
Rosenthal     & low & low (0.398) & \multirow{5}{5.0cm}{\raggedright \scriptsize\itshape All five concepts are pertinent. The most important features are absent GFAP and Synaptophysin, which determine glial vs.\ neuronal CNS tumors; when both are negative, the differential narrows to high-grade CNS or non-glial. The remaining features are helpful but non-specific---experts also rely on tumor morphology, cellular density, and growth pattern.} \\
GFAP          & low & low (0.032) & \\
Synaptophysin & low & low (0.068) & \\
mitotic       & low & low (0.363) & \\
necrosis      & low & low (0.267) & \\
\bottomrule
\end{tabular}
\end{minipage}
\caption{\textbf{Non-glial reasoning trace}.
Clinical rationale provided by an independent neuropathologist.}
\label{tab:case_ng}
\end{figure}
\vspace{-0.4em}

\section{Ablation Study}
\label{app:ablation}
\vspace{-0.4em}
\subsection{Configuration Definitions}
\vspace{-0.4em}
To rigorously evaluate the architectural components of \ours{}, we define the following configuration terminology used in our ablation study (Table~\ref{tab:ablation}):
\begin{itemize}[leftmargin=*, nosep]
    \item \textbf{Hierarchy (\texttt{flat} vs. \texttt{hier}):} \texttt{flat} uses a single-level concept bottleneck, treating concepts independently without hierarchical conditioning. \texttt{hier} employs a serial two-level hierarchy where higher-level biomarkers condition on lower-level morphology. 
    \item \textbf{Bottleneck Strictness (\texttt{standard} vs. \texttt{-hard}):} The standard variants retain the residual connection ($\gamma_{\mathrm{res}}=1$), acting as a soft bottleneck. The \texttt{-hard} variants disable the residual pathway ($\gamma_{\mathrm{res}}=0$), forcing a strict information bottleneck. Because strict bottlenecks are harder to optimize and highly sensitive to class imbalance, all \texttt{-hard} variants employ a heavy regularization recipe consisting of balanced binary cross-entropy and a weighted data sampler.
    \item \textbf{Supervision Weights:} Standard runs use $(\lambda_1=0.5, \lambda_2=0.3)$. The \texttt{low reg.} variant scales this down to $(0.2, 0.2)$, while \texttt{high reg.} applies a much stricter penalty of $(1.0, 0.6)$.
\end{itemize}

\begin{table}[htbp]
\caption{Classification results of \ours{} variants on PBT (10-fold CV, mean $\pm$ std).}
\label{tab:ablation}
\centering
\footnotesize
\begin{tabular}{@{}ll c ll@{}}
\toprule
Model Variant & Modalities & Residual & Accuracy & F1-macro \\
\midrule
\multicolumn{5}{@{}l}{\textbf{Architecture \& Hierarchy}} \\
flat (morph) & WH & \cmark & 83.29 $\pm$ 9.40 & 75.17 $\pm$ 13.45 \\
flat-hard (morph) & WH & \xmark & 83.24 $\pm$ 7.55 & \textbf{77.32} $\pm$ 11.05 \\
flat (morph+bio) & WH & \cmark & 82.74 $\pm$ 6.55 & 74.85 $\pm$ 12.28 \\
flat-hard (morph+bio) & WH & \xmark & 81.76 $\pm$ 6.38 & 75.34 $\pm$ 11.07 \\
hier (morph+bio) & WH & \cmark & \textbf{83.71} $\pm$ 6.91 & 76.88 $\pm$ 12.40 \\
hier-hard (morph+bio) & WH & \xmark & 83.21 $\pm$ 5.81 & 73.64 $\pm$ 12.47 \\
\midrule
\multicolumn{5}{@{}l}{\textbf{Supervision Weights (Applied to hier-hard)}} \\
hier-hard (low reg.) & WH & \xmark & 83.71 $\pm$ 5.27 & 75.05 $\pm$ 12.03 \\
hier-hard (high reg.) & WH & \xmark & 82.21 $\pm$ 7.69 & 73.48 $\pm$ 15.09 \\
\midrule
\multicolumn{5}{@{}l}{\textbf{Modality Ablations}} \\
hier (WSI-only) & W & \cmark & 82.24 $\pm$ 7.63 & 73.82 $\pm$ 14.84 \\
flat-hard (bio) & WH & \xmark & 83.74 $\pm$ 6.78 & 74.92 $\pm$ 14.28 \\
flat (bio) & WH & \cmark & 82.74 $\pm$ 6.49 & 72.21 $\pm$ 15.32 \\
\bottomrule
\end{tabular}
\end{table}

\noindent\textbf{Results and Analysis.} 
Table~\ref{tab:ablation} presents the ablation study. We observe three key findings:

\textbf{Hierarchical vs. Flat Concept Architecture.} When utilizing both morphology and biomarker concepts, organizing them in a serial hierarchy (\texttt{hier (morph+bio)}) outperforms both the flat fusion strategy (\texttt{flat (morph+bio)}) and the biomarker-only baseline (\texttt{flat-hard (bio)}). \texttt{hier (morph+bio)} achieves the highest overall accuracy ($83.71\%$) and a strong macro-F1 ($76.88\%$), validating that explicitly modeling the clinical dependency, where higher-level biomarkers condition on fundamental tissue morphology, yields better discriminative representations.

\textbf{The Crucial Role of Residual Pathways.} For a single-level morphology bottleneck, removing the residual and applying the stabilized hard-bottleneck recipe (\texttt{flat-hard (morph)}) acts as a strong regularizer, yielding the highest macro-F1 ($77.32\%$). However, in the two-level hierarchical setting, the strict bottleneck (\texttt{hier-hard (morph+bio)}) leads to a noticeable performance drop (F1 drops from $76.88\%$ to $73.64\%$). This confirms that as the concept hierarchy deepens, residual connections become essential to prevent severe information loss, maintaining high performance without requiring complex class-balancing tricks.

\textbf{Modality and Supervision Trade-offs.} Removing the graph modality (\texttt{hier (WSI-only)}) degrades the macro-F1 to $73.82\%$, reinforcing the necessity of multimodal fusion. Furthermore, tuning the concept loss weight is critical: applying overly strict concept supervision weights (\texttt{hier-hard (high reg.)}) severely harms task performance (F1 drops to $73.48\%$), highlighting the need for balanced concept regularization.
\vspace{-0.4em}

\section{Training-Set Subsampling Protocol}
\label{app:sample_efficiency_protocol}
\vspace{-0.4em}
\textbf{Training subsets.}
We construct training subsets at four sizes: $N \in \{50, 100, 150, 164\}$, where $N{=}164$ is the full training pool (the same pool used for one fold of the main 10-fold cross-validation). For each $N$, we draw $5$ random subsets of size $N$ from the training pool, yielding $5$ splits per training size and $20$ training runs per model in total. The validation and test sets are held fixed across training sizes within a split, so only the training-set size varies; this isolates the data-efficiency effect from changes in evaluation distribution.

\textbf{Models and modalities.}
We compare three models: \ours{}~(full), the two-layer concept-bottleneck variant with five L1 morphology and five L2 biomarker concepts; \ours{}~(bio-only), the single-layer ablation that retains only the five L2 biomarker concepts; and PathMoE~\citep{yu2026pathmoe}, the I2MoE-Transformer backbone without concept supervision. The two \ours{} variants use the WSI + Graph (WH) modality pair, while PathMoE uses the WSI + Graph + Text (WHT) triplet, this matches the configuration used for each method's best run on the full data, so our sample-efficiency comparison reflects each model's intended deployment regime rather than an artificially matched input.

\textbf{Training.}
All three models share the optimizer and schedule: AdamW at learning rate $1\mathrm{e}{-4}$, batch size $16$, $150$ epochs, dropout $0.5$, hidden dimension $256$, $16$ patches per slide. We report the test-set metric corresponding to the checkpoint with highest validation macro-F1 within the run. For each $(\text{model}, N)$ pair we report mean and standard deviation across the $5$ splits.

\vspace{-0.4em}
\section{Mutual Information Plane Details}
\label{app:mi_details}
\vspace{-0.4em}
This appendix details the mutual information estimation protocol used to produce the information plane analysis in Fig.~\ref{fig:mi_plane} and \ref{fig:mi_curve}. We describe the models compared, the estimation methodology, and the trajectory post-processing steps.
\vspace{-0.4em}
\subsection{Information Plane Framework}
\vspace{-0.4em}
The information plane~\cite{tishby2000information} visualizes the trade-off between compression and task-relevant information in learned representations. For a bottleneck representation $C$ derived from input $X$ and a target label $Y$, the plane plots:
\begin{itemize}
  \item \textbf{X-axis}: $I(X; C)$, the mutual information between the input and the bottleneck: measuring how much information about the input is preserved (capacity).
  \item \textbf{Y-axis}: $I(C; Y)$, the mutual information between the bottleneck and the label: measuring how much task-relevant information is retained.
\end{itemize}
A good concept bottleneck should compress the raw features into interpretable concepts (lower $I(X; C)$) while preserving task-relevant signal (high $I(C; Y)$). The trajectory of these quantities across training epochs reveals how the model learns to balance these objectives.
\vspace{-0.4em}
\subsection{Models and Bottlenecks Compared}
\vspace{-0.4em}
We compare three model variants on the same information plane, all trained on the WSI + graph (WH) modality combination with 10-fold cross-validation:

\begin{table}[htbp]
\caption{Summary of model representations and tracked features compared in the Information Plane analysis.}
\label{tab:mi_models}
\centering
\renewcommand{\arraystretch}{1.3}
\resizebox{\linewidth}{!}{%
\begin{tabular}{@{} l l c p{6.5cm} @{}}
\toprule
\textbf{Representation} & \textbf{Tracked Feature} & \textbf{Dim} & \textbf{Description} \\
\midrule
PathMoE & $\mathbf{z}$ (expert latent) & $768$ & Mean-pooled expert outputs; no bottleneck (upper reference). \\
\ours{} L1 & $\hat{\mathbf{c}}_1$ (concept embedding) & $5 {\times} 16 = 80$ & Morphology concept embeddings, mean over 4 experts. \\
\ours{} L2 & $\hat{\mathbf{c}}_2$ (concept embedding) & $5 {\times} 16 = 80$ & Biomarker concept embeddings, serial L1$\to$L2 design. \\
CBM-\ours{} L1 & $\mathbf{p}_1$ (scalar probability) & $19$ & Scalar concept probabilities (5 concepts, one-hot encoded). \\
CBM-\ours{} L2 & $\mathbf{p}_2$ (scalar probability) & $10$ & Scalar biomarker probabilities (5 binary concepts). \\
\bottomrule
\end{tabular}%
}
\end{table}

The CBM variant replaces the concept embedding layer with a standard scalar concept bottleneck~\cite{koh2020concept}, where only the predicted probability $p_k$ (without the embedding $\hat{\mathbf{c}}_k$) is passed to the classifier. This serves as a baseline to quantify the information retained by the embedding representation versus the scalar bottleneck.
\vspace{-0.4em}
\subsection{MI Estimation: Gaussian Parametric Estimator}
\vspace{-0.4em}
We estimate mutual information using a Gaussian parametric estimator with PCA pre-reduction. This choice balances stability under limited sample sizes ($N \approx 200$ samples per epoch across all 10 splits) with computational tractability.
\vspace{-0.4em}
\subsubsection{Estimator Definition}
\vspace{-0.4em}
For bottleneck features $C \in \mathbb{R}^{N \times d}$ and labels $Y$, the Gaussian parametric MI is:
\begin{equation}
  I(C; Y)
  = H(C) - \sum_{\ell} p(Y\!=\!\ell) \cdot H(C \mid Y\!=\!\ell),
  \label{eq:gaussian_mi}
\end{equation}

where $H(C) = \frac{1}{2}\log\big((2\pi e)^d \det(\Sigma)\big)$ is the differential entropy of a Gaussian with covariance $\Sigma$, and $H(C \mid Y\!=\!\ell)$ is the class-conditional entropy estimated from the covariance $\Sigma_{\ell}$ of samples belonging to class $\ell$. Both covariances are computed with bias correction (dividing by $N$ rather than $N-1$) for consistency with the maximum-likelihood Gaussian fit.
\vspace{-0.4em}
\subsubsection{PCA Pre-reduction and Regularization}
\vspace{-0.4em}
Direct covariance estimation in the raw feature space is unstable when $d$ is large relative to $N$. We apply PCA to reduce dimensionality before MI computation:

\begin{equation}
  k = \min\!\big(20,\; d,\; \lfloor N/4 \rfloor\big).
  \label{eq:pca_k}
\end{equation}

To handle linearly dependent dimensions (e.g., probability vectors summing to 1 in the CBM bottleneck), we further detect the \textit{effective rank}: the number of PCA eigenvalues exceeding $10\varepsilon$, where $\varepsilon = 10^{-3}$ is the Tikhonov regularization parameter. The final PCA dimension is $k' = \min(k,\, \text{effective\_rank})$, with a floor of $k' \ge 2$.

All covariance matrices are regularized with $\Sigma_{\text{reg}} = \Sigma + \varepsilon \mathbf{I}$. If the regularized log-determinant is still non-positive, the entropy is clamped to 0. The X-axis quantity $I(X; C)$ is upper-bounded by $H(C)$, since $H(C \mid X) = 0$ for deterministic concept predictors.

All models were trained for 150 epochs with early stopping (patience 30), batch size 16, learning rate $10^{-4}$, and dump interval 1 (every epoch). Dumps from all 10 splits are pooled per epoch for MI computation, yielding approximately $N=200$ samples per epoch.
\vspace{-0.4em}
\subsection{Trajectory Post-processing}
\vspace{-0.4em}
\subsubsection{Truncation}
\vspace{-0.4em}
To focus on the learning phase and remove post-convergence noise, trajectories are truncated at the epoch of peak $I(C; Y)$ plus a 3-epoch buffer. For CEM variants, an additional monotonic filter retains only points where \textit{both} $H(C)$ and $I(C; Y)$ are strictly increasing relative to the previous retained point, suppressing training zigzag while preserving the dominant upward trend. The CBM variant is not truncated (its $I(C; Y)$ peaks at epoch 78, and truncation would conceal the fact that the scalar bottleneck requires approximately 65 epochs to produce any class-relevant concept signal).
\vspace{-0.4em}
\subsubsection{Smoothing}
\vspace{-0.4em}
Gaussian filter smoothing ($\sigma=2.0$ epochs) is applied to both $H(C)$ and $I(C; Y)$ trajectories before visualization. Epoch-proportional marker sizes follow the convention of~\cite{espinosa2022concept}.
\vspace{-0.4em}
\subsection{Key Observations}
\vspace{-0.4em}
The information plane analysis (Fig.~\ref{fig:mi_plane} and \ref{fig:mi_curve}) yields three main findings:
\begin{enumerate}
  \item \textbf{CEM embeddings preserve task information comparable to raw latents.} The converged $I(\hat{\mathbf{c}}; Y)$ of \ours{} L1 and L2 (12.3--14.2 nats) is comparable to or slightly exceeds that of PathMoE's raw expert latents (11.6 nats). This confirms that the concept embedding layer, with residual pathways, does not act as a lossy bottleneck.

  \item \textbf{Scalar bottlenecks lose substantial information.} The CBM variant achieves at most 3.3 nats (L1) and 0.7 nats (L2) of label-relevant information, a 70--95\% reduction relative to \ours{}. This quantifies the cost of forcing the classifier to rely solely on scalar concept probabilities.

  \item \textbf{Scalar bottlenecks learn class-relevant concepts slowly.} CBM $I(\mathbf{p}_1; Y)$ remains at zero for the first 64 epochs, indicating that the scalar concept predictor learns to produce confident concept predictions long before those predictions become class-discriminative. This slow emergence of task-relevant concept structure is invisible to standard validation metrics (which plateau much earlier) and highlights a hidden training dynamic revealed by the information plane.
\end{enumerate}

These findings provide quantitative support for the architectural choices in \ours{}: concept embeddings with residual pathways preserve the flexibility of unconstrained multimodal fusion while organizing the representation around clinically interpretable axes.


\vspace{-0.4em}
\section{Additional Method Details and Proofs}
\label{app:method_details}
\vspace{-0.4em}
\subsection{Concept annotation and masking details}
\label{app:concept_annotation}
\vspace{-0.4em}
The method uses two levels of clinical concepts. For loss notation,
$\mathcal{K}_1=\{1,\ldots,K_1\}$ indexes morphology concepts and
$\mathcal{K}_2=\{1,\ldots,K_2\}$ indexes biomarker concepts within their own
levels. When global concept indices are used in the architecture, biomarker
concept $j\in\mathcal{K}_2$ corresponds to global index $K_1+j$.

Morphology concepts are extracted from pathology reports and represented as
soft-ordinal targets $t^1_{n,k}\in[0,1]$. A target is treated as observed when
the report provides explicit evidence about the concept; unmentioned concepts
are assigned mask $m^1_{n,k}=0$ and do not contribute to the concept loss.
Biomarker concepts are obtained from available protein-level or molecular
biomarker annotations and represented as binary targets
$t^2_{n,j}\in\{0,1\}$, with mask $m^2_{n,j}=1$ only when the corresponding
annotation is available. The level-specific concept loss in
Eq.~\eqref{eq:concept_loss} applies the soft-ordinal loss $\ell_1$ to
morphology targets and binary cross-entropy $\ell_2$ to biomarker targets. The
same masks are used for all experts, so missing concepts neither penalize nor
reward any expert pathway.
\vspace{-0.4em}
\subsection{Encoder and interaction-specialization details}
\label{app:interaction_details}
\vspace{-0.4em}
For image inputs, patch-level features are aggregated into an image-level
representation with gated-attention multiple-instance learning~\cite{lu2021data}.
For graph-structured inputs, node features are encoded by a message-passing
network such as GraphSAGE~\cite{hamilton2017inductive} and pooled into a
graph-level representation. All modality embeddings are projected to the common
dimension $d$ before entering the expert layers.

For the interaction-specialization loss in Eq.~\eqref{eq:interaction_loss},
modality $m$ is perturbed as
\begin{equation*}
  \tilde{\mathbf{e}}_m
  =
  \mathbf{e}_m+\boldsymbol{\epsilon},
  \qquad
  \boldsymbol{\epsilon}\sim\mathcal{N}(\mathbf{0},\sigma^2\mathbf{I}).
  \label{eq:app_modality_perturbation}
\end{equation*}
Let $\mathbf{p}_e^{\mathrm{clean}}$ denote the expert prediction with the
original embeddings, and let $\mathbf{p}_e^{m}$ denote the prediction after
perturbing modality $m$. The sign pattern is
\begin{equation*}
  s_e^m
  =
  \begin{cases}
  -1, & e=U_m,\\
  -1, & e=S,\\
  +1, & e=R,\\
  +1, & e=U_j,\ j\ne m.
  \end{cases}
  \label{eq:app_interaction_signs}
\end{equation*}
Thus the uniqueness expert $U_m$ is encouraged to be sensitive to modality
$m$ and invariant to the other modalities, the redundancy expert $R$ is
encouraged to be invariant to single-modality perturbations, and the synergy
expert $S$ is encouraged to be sensitive to all modalities. The averaged
interaction loss is
\begin{equation*}
  \mathcal{L}_{\mathrm{int}}
  =
  \frac{1}{E}\sum_{e\in\mathcal{E}}\mathcal{L}_{\mathrm{int}}^e .
  \label{eq:app_interaction_average}
\end{equation*}
\vspace{-0.4em}
\subsection{Hierarchical concept block details}
\label{app:concept_hierarchy_details}
\vspace{-0.4em}
The main text defines the first-level concept block from the expert latent
$\mathbf{z}_e$. For morphology concept $k\le K_1$, the positive state,
negative state, scalar activation, and concept embedding are computed as in
Eqs.~\eqref{eq:positive_negative_concepts}--\eqref{eq:soft_concept_embedding}.
The first-level embeddings are collected in $\mathbf{B}_{e,1}$.

For biomarker concept $j\in\mathcal{K}_2$, the second-level input is the
concatenated representation defined as $\mathbf{z}_{e,2}=[\mathbf{z}_e;\mathbf{B}_{e,1}]$. The biomarker concept block
uses the same positive-negative embedding construction, with projections defined
on the second-level input:
\begin{equation*}
  \mathbf{c}_{e,K_1+j}^{+}
  =
  \mathrm{LeakyReLU}(\phi_{e,K_1+j}^{+}(\mathbf{z}_{e,2})),
  \qquad
  \mathbf{c}_{e,K_1+j}^{-}
  =
  \mathrm{LeakyReLU}(\phi_{e,K_1+j}^{-}(\mathbf{z}_{e,2})),
  \label{eq:app_l2_positive_negative}
\end{equation*}
\begin{equation*}
  p_{e,K_1+j}
  =
  \sigma\!\left(
    s_{e,K_1+j}([\mathbf{c}_{e,K_1+j}^{+};\mathbf{c}_{e,K_1+j}^{-}])
  \right),
  \label{eq:app_l2_activation}
\end{equation*}
\begin{equation*}
  \hat{\mathbf{c}}_{e,K_1+j}
  =
  p_{e,K_1+j}\mathbf{c}_{e,K_1+j}^{+}
  +(1-p_{e,K_1+j})\mathbf{c}_{e,K_1+j}^{-}
  +\gamma_{\mathrm{res}}\,\psi_{e,K_1+j}(\mathbf{z}_{e,2}).
  \label{eq:app_l2_embedding}
\end{equation*}
The second-level embeddings are then collected in $\mathbf{B}_{e,2}$ as in
Eq.~\eqref{eq:l2_concepts}. If a concept level is unavailable for a cohort, the
corresponding block and loss are removed by setting the associated $K_q$ to
zero.
\vspace{-0.4em}
\subsection{Concept-wise attribution details}
\label{app:attribution_details}
\vspace{-0.4em}
For analyses that summarize expert-pathway support, let $\alpha_e$ denote the
sample-specific expert weight used to combine expert predictions. If an
implementation uses uniform expert averaging, set $\alpha_e=1/E$. For concept
$k$, the class-agnostic pathway evidence profile is
\begin{equation*}
  a_{e,k}
  =
  \frac{\alpha_e p_{e,k}}
  {\sum_{e'\in\mathcal{E}}\alpha_{e'}p_{e',k}+\varepsilon},
  \qquad e\in\mathcal{E}.
  \label{eq:app_concept_evidence_profile}
\end{equation*}
This profile decomposes the predicted activation of concept $k$ across
uniqueness, redundancy, and synergy expert pathways.

For class-specific explanations, let $c$ be the predicted class or a
user-specified class, and let $R_e$ denote the representation passed to the
expert head $g_e$. Let $R_e^{\neg k}$ be the
same representation with $\hat{\mathbf{c}}_{e,k}$ replaced by a neutral concept
state. The neutral state is fixed across samples, such as the negative concept
embedding or the empirical mean concept embedding. The logit-ablation score is
\begin{equation*}
  \Delta_{e,k}^{(c)}
  =
  \alpha_e
  \left[
    g_{e,c}(R_e)-g_{e,c}(R_e^{\neg k})
  \right],
  \label{eq:class_specific_attribution}
\end{equation*}
where $g_{e,c}$ is the $c$-th logit of expert $e$. Positive values indicate
supporting evidence for class $c$, and negative values indicate suppressive
evidence. For visualization, positive evidence can be normalized across
pathways:
\begin{equation*}
  \bar{\Delta}_{e,k}^{(c)}
  =
  \frac{\max(\Delta_{e,k}^{(c)},0)}
  {\sum_{e'\in\mathcal{E}}\max(\Delta_{e',k}^{(c)},0)+\varepsilon}.
  \label{eq:app_normalized_class_attribution}
\end{equation*}
The score $a_{e,k}$ measures which pathway activates a concept, while
$\Delta_{e,k}^{(c)}$ measures how the pathway-specific concept changes the
diagnostic logit.
\vspace{-0.4em}
\subsection{Proof of Proposition~\ref{prop:soft_concept_augmentation}}
\label{app:proof_soft_concept_augmentation}
\vspace{-0.4em}
\begin{proof}
For a fixed expert, $R_{e,1}$ contains $R_{e,0}=\mathbf{z}_e$ as a coordinate
subvector, and $R_{e,2}$ contains $R_{e,1}$ as a coordinate subvector. Hence the
sigma-algebra generated by $R_{e,0}$ is contained in that generated by
$R_{e,1}$, which is contained in that generated by $R_{e,2}$. By the chain rule
for mutual information,
\begin{equation*}
  I(R_{e,1};Y)
  =
  I(R_{e,0};Y)
  +
  I([\mathbf{B}_{e,1};\mathbf{r}_{e,1}];Y\mid R_{e,0})
  \ge
  I(R_{e,0};Y),
\end{equation*}
and similarly
\begin{equation*}
  I(R_{e,2};Y)
  =
  I(R_{e,1};Y)
  +
  I(\mathbf{B}_{e,2};Y\mid R_{e,1})
  \ge
  I(R_{e,1};Y).
\end{equation*}
This proves the mutual-information ordering.

For unrestricted population log-loss,
\begin{equation*}
  \mathcal{R}_{\log}^{\star}(R)
  =
  \inf_q\mathbb{E}[-\log q(Y\mid R)]
  =
  H(Y\mid R)
  =
  H(Y)-I(R;Y),
\end{equation*}
where the optimum is attained by the true conditional distribution
$q(y\mid R)=\mathbb{P}(Y=y\mid R)$. The mutual-information ordering therefore
implies
\begin{equation*}
  \mathcal{R}_{\log}^{\star}(R_{e,2})
  \le
  \mathcal{R}_{\log}^{\star}(R_{e,1})
  \le
  \mathcal{R}_{\log}^{\star}(R_{e,0}).
\end{equation*}
If the appended concept and residual coordinates are deterministic functions of
$\mathbf{z}_e$ and $\mathbf{z}_e$ is retained, then $R_{e,j}$ and
$\mathbf{z}_e$ generate the same sigma-algebra. Thus
$I(R_{e,j};Y)=I(\mathbf{z}_e;Y)$ for $j\in\{0,1,2\}$.

For the finite-capacity predictor classes, any head operating on $R_{e,0}$ can
be represented as a head operating on $R_{e,1}$ by setting the weights on the
appended coordinates to zero. The same argument embeds any head on $R_{e,1}$
into the head class on $R_{e,2}$. Therefore
$\mathcal{F}_0\subseteq\mathcal{F}_1\subseteq\mathcal{F}_2$, and taking the
infimum over a larger class cannot increase empirical classification loss:
\begin{equation*}
  \inf_{f\in\mathcal{F}_2}\widehat{\mathcal{L}}_{\mathrm{cls}}(f)
  \le
  \inf_{f\in\mathcal{F}_1}\widehat{\mathcal{L}}_{\mathrm{cls}}(f)
  \le
  \inf_{f\in\mathcal{F}_0}\widehat{\mathcal{L}}_{\mathrm{cls}}(f).
\end{equation*}
\end{proof}
\vspace{-0.4em}
\subsection{Proof of Theorem~\ref{thm:concept_localization}}
\label{app:proof_concept_localization}
\vspace{-0.4em}
\begin{proof}
Let $\mathcal{G}_{\tau}=\ell_{\mathrm{cls}}\circ\mathcal{H}_{\tau}$. By the
standard Rademacher uniform convergence bound for bounded losses, with
probability at least $1-\delta$,
\begin{equation}
  \sup_{h\in\mathcal{H}_{\tau}}
  \left|
  \mathcal{R}_{\mathrm{cls}}(h)
  -
  \widehat{\mathcal{R}}_{\mathrm{cls}}(h)
  \right|
  \le
  2\mathfrak{R}_n(\mathcal{G}_{\tau})
  +
  B_{\ell}\sqrt{\frac{\log(2/\delta)}{2n}}.
  \label{eq:app_uniform_convergence}
\end{equation}
Let the right-hand side be $\Delta_{\tau}$. For any
$h\in\mathcal{H}_{\tau}$, empirical optimality of $\hat h_{\tau}$ gives
\begin{equation*}
  \widehat{\mathcal{R}}_{\mathrm{cls}}(\hat h_{\tau})
  \le
  \widehat{\mathcal{R}}_{\mathrm{cls}}(h).
\end{equation*}
Using Eq.~\eqref{eq:app_uniform_convergence} on $\hat h_{\tau}$ and on the
comparator $h$ yields
\begin{equation*}
  \mathcal{R}_{\mathrm{cls}}(\hat h_{\tau})
  \le
  \widehat{\mathcal{R}}_{\mathrm{cls}}(\hat h_{\tau})+\Delta_{\tau}
  \le
  \widehat{\mathcal{R}}_{\mathrm{cls}}(h)+\Delta_{\tau}
  \le
  \mathcal{R}_{\mathrm{cls}}(h)+2\Delta_{\tau}.
\end{equation*}
Taking the infimum over $h\in\mathcal{H}_{\tau}$ gives
\begin{equation*}
  \mathcal{R}_{\mathrm{cls}}(\hat h_{\tau})
  \le
  \inf_{h\in\mathcal{H}_{\tau}}\mathcal{R}_{\mathrm{cls}}(h)
  +
  4\mathfrak{R}_n(\ell_{\mathrm{cls}}\circ\mathcal{H}_{\tau})
  +
  2B_{\ell}\sqrt{\frac{\log(2/\delta)}{2n}}.
\end{equation*}
Since
$\mathcal{H}_{\tau}\subseteq\mathcal{H}$, monotonicity of Rademacher complexity
implies
\begin{equation*}
  \mathfrak{R}_n(\ell_{\mathrm{cls}}\circ\mathcal{H}_{\tau})
  \le
  \mathfrak{R}_n(\ell_{\mathrm{cls}}\circ\mathcal{H}).
\end{equation*}
Substituting the near-sufficiency condition from
Theorem~\ref{thm:concept_localization} gives the stated finite-sample
comparison.
\end{proof}
\vspace{-0.4em}
\subsection{Auxiliary-gradient contraction result}
\label{app:proof_aux_gradient_contraction}
\vspace{-0.4em}
The localization result describes the statistical role of concept supervision.
The next proposition gives the complementary optimization statement used in
Section~\ref{sec:training_objective}.

\begin{proposition}[Aligned auxiliary gradients improve task-loss contraction]
\label{prop:aux_gradient_contraction}
Let
\begin{equation*}
  \mathcal{L}_a(\theta)
  =
  \lambda_1\mathcal{L}_{\mathrm{concept}}^{\mathrm{L1}}(\theta)
  +
  \lambda_2\mathcal{L}_{\mathrm{concept}}^{\mathrm{L2}}(\theta)
  +
  \lambda_{\mathrm{int}}\mathcal{L}_{\mathrm{int}}(\theta),
  \label{eq:app_auxiliary_loss}
\end{equation*}
and define
$g_y=\nabla\mathcal{L}_{\mathrm{cls}}(\theta)$ and
$g_a=\nabla\mathcal{L}_{a}(\theta)$. Consider the update
$\theta^+=\theta-\eta(g_y+g_a)$. Assume that along the training trajectory:
(i) $\mathcal{L}_{\mathrm{cls}}$ is $\beta$-smooth; (ii)
$\mathcal{L}_{\mathrm{cls}}$ satisfies the Polyak--\L{}ojasiewicz (PL)
condition
\begin{equation}
  \|g_y\|^2
  \ge
  2\mu\bigl(\mathcal{L}_{\mathrm{cls}}(\theta)-\mathcal{L}_{\mathrm{cls}}^{\star}\bigr);
  \label{eq:app_pl_condition}
\end{equation}
(iii) the auxiliary gradient is task-aligned,
\begin{equation*}
  \langle g_y,g_a\rangle
  \ge
  \rho\|g_y\|^2,
  \qquad \rho\ge0;
  \label{eq:app_gradient_alignment}
\end{equation*}
(iv) the joint gradient norm is controlled,
\begin{equation*}
  \|g_y+g_a\|^2
  \le
  B_{\nabla}\|g_y\|^2;
  \label{eq:app_gradient_norm_control}
\end{equation*}
and (v) $1+\rho-\beta\eta B_{\nabla}/2>0$. Then
\begin{equation}
  \mathcal{L}_{\mathrm{cls}}(\theta^+)-\mathcal{L}_{\mathrm{cls}}^{\star}
  \le
  \left[
  1
  -
  2\eta\mu
  \left(
  1+\rho-
  \frac{\beta\eta B_{\nabla}}{2}
  \right)
  \right]
  \bigl(\mathcal{L}_{\mathrm{cls}}(\theta)-\mathcal{L}_{\mathrm{cls}}^{\star}\bigr).
  \label{eq:aux_contraction_bound}
\end{equation}
Moreover, this contraction is sharper than the task-only update whenever
\begin{equation}
  \rho
  >
  \frac{\beta\eta}{2}(B_{\nabla}-1).
  \label{eq:app_aux_faster_condition}
\end{equation}
\end{proposition}

\begin{proof}
By $\beta$-smoothness of $\mathcal{L}_{\mathrm{cls}}$,
\begin{align*}
  \mathcal{L}_{\mathrm{cls}}(\theta^+)
  &\le
  \mathcal{L}_{\mathrm{cls}}(\theta)
  -
  \eta\langle g_y,g_y+g_a\rangle
  +
  \frac{\beta\eta^2}{2}\|g_y+g_a\|^2 \\
  &\le
  \mathcal{L}_{\mathrm{cls}}(\theta)
  -
  \eta(1+\rho)\|g_y\|^2
  +
  \frac{\beta\eta^2B_{\nabla}}{2}\|g_y\|^2 \\
  &=
  \mathcal{L}_{\mathrm{cls}}(\theta)
  -
  \eta
  \left(
  1+\rho-
  \frac{\beta\eta B_{\nabla}}{2}
  \right)
  \|g_y\|^2 .
\end{align*}
Applying the PL condition in Eq.~\eqref{eq:app_pl_condition} gives
Eq.~\eqref{eq:aux_contraction_bound}.

For the task-only update $\theta_{\mathrm{task}}^+=\theta-\eta g_y$, the same
argument gives the contraction coefficient
\begin{equation*}
  1-2\eta\mu\left(1-\frac{\beta\eta}{2}\right).
\end{equation*}
The auxiliary update is sharper when
\begin{equation*}
  1+\rho-
  \frac{\beta\eta B_{\nabla}}{2}
  >
  1-
  \frac{\beta\eta}{2},
\end{equation*}
which is equivalent to Eq.~\eqref{eq:app_aux_faster_condition}.
\end{proof}

\vspace{-0.4em}
\section{Broader Impacts}
\label{app:broader_impacts}
\vspace{-0.4em}
The deployment of multimodal AI in computational pathology carries significant societal and clinical implications. 

\textbf{Positive Impacts.} 
\ours{} directly addresses the critical need for transparency in clinical AI. By explicitly mapping heterogeneous inputs (e.g., whole-slide images, pathology reports, and molecular biomarkers) to understandable morphology and biomarker concepts, \ours{} aligns model decisions with established pathological reasoning. This interpretable design helps bridge the trust gap between black-box deep learning models and medical practitioners. In complex diagnostic domains where decisions are highly reliant on multimodal synthesis, \ours{} offers a scalable, transparent assistive tool that can help alleviate the workload of clinicians and reduce diagnostic variability.

\textbf{Potential Risks.} 
Despite its transparent design, \ours{} is not immune to the risks inherent in medical AI. First, there is a risk of \textit{automation bias}, where clinicians might over-rely on the extracted concepts or the model's final predictions without independent verification, particularly if the model presents its rationale confidently. Second, as with many data-driven models, \ours{} is susceptible to \textit{representation bias}. If the institutional cohorts used for training lack demographic or geographic diversity, the learned concept representations may not generalize well to underrepresented populations, potentially exacerbating existing healthcare disparities.

\textbf{Mitigations and Safe Deployment.} 
To mitigate these risks, we emphasize that \ours{} is designed strictly as an \textit{assistive clinical decision support tool}, not an autonomous diagnostic agent. The concept-level attributions are intended to surface relevant evidence for human review, keeping the clinician firmly in the loop. Furthermore, before any clinical deployment, \ours{} must undergo rigorous, multi-center prospective validation to ensure robustness across diverse patient demographics, varying tissue preparation protocols, and different hardware scanners. We encourage future work to explicitly evaluate the fairness and calibration of concept-guided models across diverse sub-populations.